\documentclass[twoside,11pt]{article}

\usepackage{jmlr2e}

\newtheorem{mytheo}{Theorem}

\usepackage[show]{ed}
\usepackage{amsmath}
\usepackage{amsfonts}
\usepackage{amssymb}
\usepackage{graphicx}
\usepackage{verbatim}
\usepackage{mdwlist}
\usepackage{topcapt}

\usepackage{color}
\usepackage{xcolor,colortbl}
\usepackage{booktabs}
\usepackage{url}
\usepackage{caption}
\usepackage{subfig}
\usepackage{enumerate}
\usepackage{comment}

\usepackage{algorithm}
\usepackage[noend]{algpseudocode}

\newtheorem{mycol}{Corollary}

\algnewcommand\algorithmicinput{\textbf{INPUT:}}
\algnewcommand\INPUT{\item[\algorithmicinput]}

\newcommand{\lsize}{n}
\newcommand{\osize}{m}
\newcommand{\qsize}{q}
\newcommand{\odsize}{d}
\newcommand{\qdsize}{r}
\newcommand{\ispace}{\mathcal{X}}
\newcommand{\taskspace}{\mathcal{T}}
\newcommand{\objspace}{\mathcal{D}}
\newcommand{\anyspace}{\mathcal{X}}

\newcommand{\kernelf}{k}
\newcommand{\gkernelf}{g}
\newcommand{\kkernelf}{\Gamma}
\newcommand{\dkernelm}{\bm{K}}
\newcommand{\tkernelm}{\bm{G}}
\newcommand{\kkernelm}{\bm{\Gamma}}
\newcommand{\regparam}{\lambda}
\newcommand{\idmatrix}{\bm{I}}
\newcommand{\trace}{\textnormal{tr}}
\newcommand{\transpose}{^\top}
\newcommand{\bm}[1]{\mathbf{#1}}
\newcommand{\ve}{\textnormal{vec}}
\newcommand{\mat}{\textnormal{mat}}
\newcommand{\diagv}{\textnormal{diag}_v}
\newcommand{\diagm}{\textnormal{diag}_m}

\newcommand{\predfun}{f}
\newcommand{\filterfun}{\varphi}
\newcommand{\objset}{D}
\newcommand{\taskset}{T}
\newcommand{\labelvec}{\bm{y}}

\newcommand{\eq}[1]{\begin{align*}#1\end{align*}}
\newcommand{\eqn}[1]{\begin{align}#1\end{align}}

\DeclareMathOperator*{\argmin}{argmin}

\firstpageno{1}

\begin{document}

\title{Efficient Pairwise Learning Using Kernel Ridge Regression:\\an Exact Two-Step Method}

\author{\name Michiel Stock \email michiel.stock@ugent.be \\
       \addr KERMIT, Department of Mathematical Modelling, Statistics and Bioinformatics\\ Ghent University\\
Coupure links 653, B-9000 Ghent, Belgium 
       \AND
       \name Tapio Pahikkala \email aatapa@utu.fi \\
       \name Antti Airola \email ajairo@utu.fi \\
       \addr University of Turku and Turku Centre for Computer Science,\\ Joukahaisenkatu 3-5 B, FIN-20520, Turku, Finland    
       \AND  
       \name Bernard De Baets \email bernard.debaets@ugent.be \\
\name Willem Waegeman \email willem.waegeman@ugent.be \\
       \addr KERMIT, Department of Mathematical Modelling, Statistics and Bioinformatics\\ Ghent University\\
Coupure links 653, B-9000 Ghent, Belgium         }

\editor{}

\maketitle

\begin{abstract}
Pairwise learning or dyadic prediction concerns the prediction of properties for pairs of objects. It can be seen as an umbrella covering various machine learning problems such as matrix completion, collaborative filtering, multi-task learning, transfer learning, network prediction and zero-shot learning. In this work we analyze kernel-based methods for pairwise learning, with a particular focus on a recently-suggested two-step method. 
We show that this method offers an appealing alternative for commonly-applied Kronecker-based methods that model dyads by means of pairwise feature representations and pairwise kernels. In a series of theoretical results, we establish correspondences between the two types of methods in terms of linear algebra and spectral filtering, and we analyze their statistical consistency. In addition, the two-step method allows us to establish novel algorithmic shortcuts for efficient training and validation on very large datasets. Putting those properties together, we believe that this simple, yet powerful method can become a standard tool for many problems. Extensive experimental results for a range of practical settings are reported. 

\end{abstract}


\section{Introduction to pairwise prediction}

Many real-world machine learning problems can be naturally represented as pairwise learning or dyadic prediction problems. In contrast to more traditional learning settings, the goal here consists of making predictions for pairs of objects, each of them being characterized by a feature representation. Amongst others,  applications of that kind emerge in biology (e.g.\ predicting mRNA-miRNA interactions), medicine (e.g.\ design of personalized drugs), chemistry (e.g. prediction of binding between two types of molecules), ecology (e.g.\ prediction of host-parasite interactions), social network analysis (e.g.\ finding links between persons) and recommender systems (e.g.\ recommending personalized products to users). 

Pairwise learning has strong connections with many other machine learning settings. Especially a link with multi-task learning can be advocated, by calling the first object of the dyad an `instance' and the second object a `task'. As a typical multi-task learning example, consider ten schools providing the grades for a selection of their students. Suppose we want to predict the scores for new students in a school. The most straightforward path would be to fit a different model for each school, which uses features of the students. Intuitively, we might do better by building one model which takes all information into account. This can be done by learning a general model which has both the student and the school as input. The prediction function is thus pairwise: it takes both an instance and a task as input. In multi-task learning, the underlying idea for making the distinction between instances and tasks is that the feature description of the instances is often considered as more informative for making a prediction, while the feature description for the tasks is mainly used to steer learning into the right direction. In the majority of multi-task learning methods, a feature description for tasks is even not given, though often the idea that the tasks can be clustered or are located on a low-dimensional manifold is exploited.


In this work we adopt a multi-task learning terminology for pairwise learning. Formally, the training set is assumed to consist of a set $S=\{(\bm{d}_h, \bm{t}_h, y_h) \mid h=1,\ldots, \lsize\}$ of labeled instance-task pairs. As such, each training input is a labeled dyad $(\bm{d}_h, \bm{t}_h)$, where $\bm{d}_h\in\objspace$ and $\bm{t}_h\in\taskspace$ are the feature representations of the instances and the tasks, respectively, and $y_h$ is the corresponding label that is either continuous or discrete. Similarly, $\objspace$ and $\taskspace$ are the corresponding spaces of instances and tasks. In pairwise learning a model $\predfun(\bm{d},\bm{t})$ is trained to make predictions for (possibly new) instances and tasks. 

Various types of learning methods can be used to solve prediction problems of that kind. Kernel methods are very popular in bioinformatics, as indicated by the large number of applications in bioinformatics. They can be easily employed for pairwise learning by defining so-called pairwise kernels $\Gamma((\bm{d}, \bm{t}), (\bar{\bm{d}}, \bar{\bm{t}}))$, which measure the similarity between two dyads $(\bm{d}, \bm{t})$ and $(\bar{\bm{d}}, \bar{\bm{t}})$. Kernels of that kind can be used in tandem with any conventional kernelized learning algorithm, such as support vector machines, kernel ridge regression and kernel Fisher discriminant analysis. In case of regression or classification tasks, one obtains prediction functions of the following form:
\eq{
\predfun(\bm{d},\bm{t}) = \sum_{h=1}^\lsize \alpha_k\Gamma((\bm{d}, \bm{t}), ({\bm{d}}_h, {\bm{t}}_h))\,,
}
with $\alpha_1,\ldots,\alpha_\lsize$ the dual parameters which are obtained by the learning algorithm -- see subsequent section for references and details. 

In this work we analyze a simple, yet elegant two-step method as an alternative for models based on pairwise kernels. In the first step, a kernel ridge regression model is trained on auxiliary tasks, and adopted to predict labels for the related target task. Then, in a second step, a second model is constructed, by employing kernel ridge regression on the target data, augmented with the predictions of the first phase. We first presented this method in a conference paper~\citep{Pahikkala2014}. One year later, the same method was independently proposed as a tool to solve zero-shot learning problems~\citep{Romera-paredes2015}. From a different but related perspective, \cite{Schrynemackers2015} recently also proposed a similar method for biological network inference, but here tree-based methods instead of kernel methods were used as base learners. Those three papers have confirmed that the two-step method can obtain state-of-the-art performance on standard benchmarks in various pairwise learning settings. In this work, we will not put the emphasis on demonstrating the performance of this method. In contrast, our main focus is rather on a theoretical, computational and experimental justification of the two-step method, since we believe that it has a lot of potential. It will become clear that our two-step method is much simpler to implement than efficient Kronecker models, while manifesting more flexible model selection capabilities. It will also be shown to be applicable to more heterogeneous transfer learning settings.

This work can be subdivided in several parts that describe different aspects of the two-step method. Before going into mathematical details, Section~\ref{related_settings} gives an overview of related settings in pairwise learning. We identify four different prediction settings, for which it is crucial to make a subdivision further in this paper. In Section~\ref{dyadicprediciton} different kernel ridge regression-based methods that can be used for pairwise learning are described, including the two-step method. Subsequently, we show in Section~\ref{theoreticalconsiderations} via linear algebra and spectral filtering that this approach is closely related to other kernel methods, but with slightly different regularization mechanisms. We also formally prove  universal approximation properties for two-step kernel ridge regression. In Section~\ref{computationalaspects} we use a specific decomposition to derive novel algorithmic shortcuts for leave-one-out cross-validation and for updating existing models with new training instances or tasks. A very important merit of the two-step method will be that there are closed-form shortcuts for cross-validation for any of the prediction settings discussed in Section~\ref{related_settings}, in contrast to the other methods. In the experiments (Section~\ref{experimentalsection}) we consider several dyadic prediction problems, studying generalization for the different settings as well as efficient implementations for training and testing the models. Our results show that the two-step method can be highly beneficial when there is no labeled data at all, or only a small amount of labeled data available for the target task, while in settings where there is a significant amount of labeled data available for the target task, an independent-task model suffices. Furthermore, we showcase the tools for performing cross-validation for the different settings and updating the model with extra instances or tasks. 

\section{The different prediction settings in pairwise learning}\label{related_settings}

In this section we make an important distinction between several types of pairwise learning settings. This will allow us to give a brief overview of related methods, and to perform a more detailed analysis in the upcoming sections. In pairwise learning it is extremely important to implement appropriate training and evaluation procedures. For example, in a large-scale meta-study about biological network identification it was found that these concepts are vital to correctly evaluate pairwise learning models~\citep{Park2012}. Given $\bm{d}$ and $\bm{t}$ as the feature representations of the instances and tasks, respectively, four settings can be distinguished:
\begin{itemize}
\item {\bf Setting A}: Both $\bm{d}$ and $\bm{t}$ are observed during training, as parts of different dyads, but the label of the dyad $(\bm{d},\bm{t})$ must be predicted. 
\item {\bf Setting B}: Only $\bm{t}$ is known during training, while $\bm{d}$ is not observed in any dyad, and the label of the dyad $(\bm{d},\bm{t})$ must be predicted.  
\item {\bf Setting C}: Only $\bm{d}$ is known during training, while $\bm{t}$ is not observed in any dyad, and the label of the dyad $(\bm{d},\bm{t})$ must be predicted.  
\item {\bf Setting D}: Neither $\bm{d}$ nor $\bm{t}$ occur in any training dyad, but the label of the dyad $(\bm{d},\bm{t})$ must be predicted. 
\end{itemize}

\begin{figure}[t]
   \centering
   \includegraphics[width=0.8\textwidth]{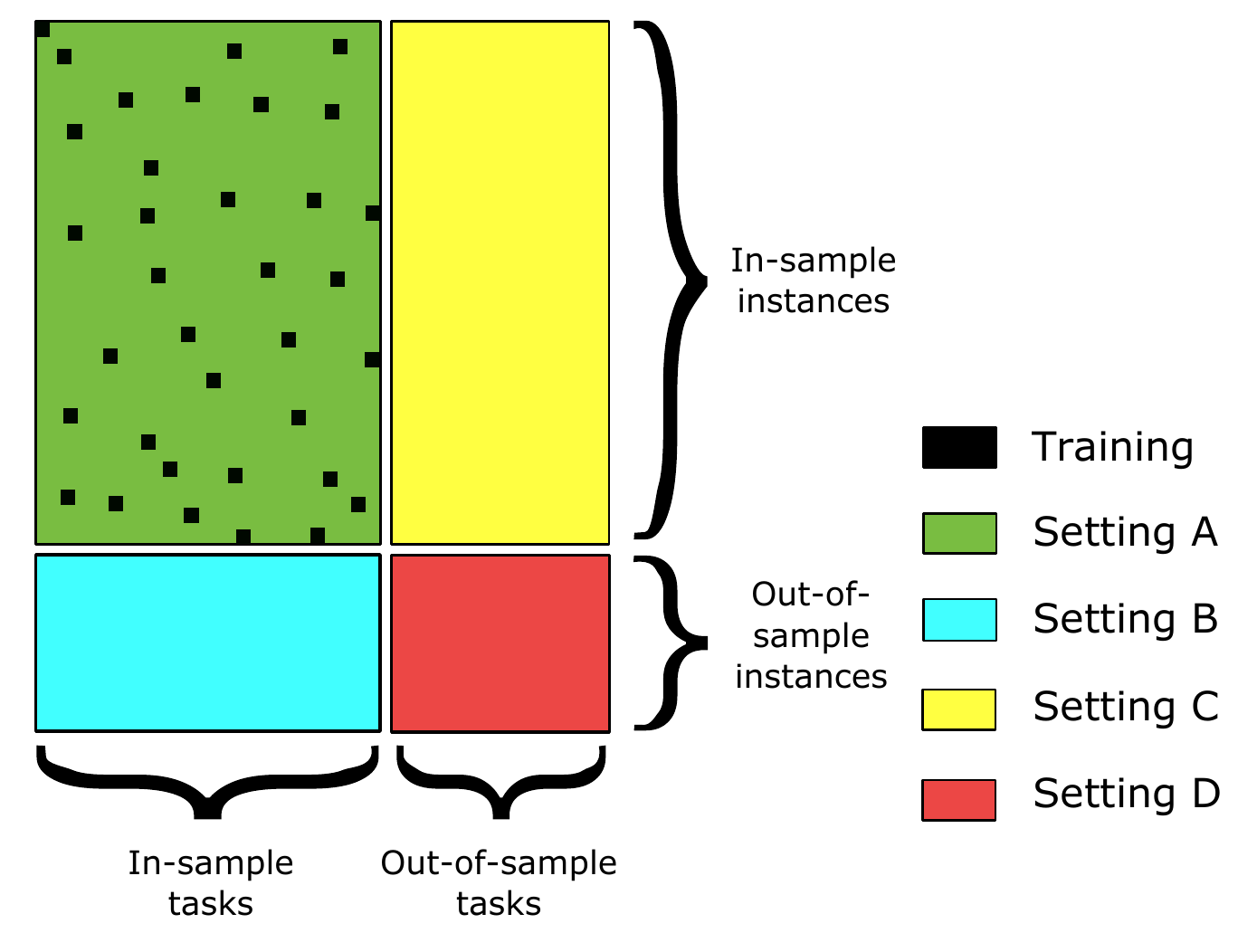} 
   \caption{An overview of different prediction settings in pairwise learning. We will mainly deal with complete datasets, where for each combination of train instance and task we have exactly one observation in the training data. See main text for links with different machine learning settings.}
   \label{conceptrelations}
\end{figure}

These settings are represented graphically in Figure~\ref{conceptrelations}. Setting A, the matrix completion problem, is of all four settings by far the most studied setting in the machine learning literature. Motivated by applications in collaborative filtering and link prediction, matrix factorization and related techniques are often applied to complete partially observed matrices. Missing values represent $(\bm{d},\bm{t})$ combinations that are not observed during training - see e.g. \cite{Larochelle2008} for a review. Many of these matrix completion algorithms do not incorporate side features (features of the instances and tasks) and make assumptions on the structure of the true label matrix by, for example, assuming that the completed matrix is low rank or has a low nuclear norm~\citep{Candes2008,Mazumder2010}. Recently, a framework based on bipartite graphs was proposed, which exploits the network structure for transductive link prediction~\citep{Liu2015a}. If one uses the notion of `latent features' (which can be implemented by means of a delta kernel), one can exploit both the structure of the label matrix as well as the side-features. Some interesting work has been done to unify these two approaches, using both the structure of the matrix, as well as features, e.g.~\cite{Basilico2004,Abernethy2008,Menon2011}. 

Settings B and C are very similar, and a variety of machine learning methods can be applied to these settings. From a recommender systems viewpoint, those settings resemble the cold start problem (new user or new item), for which hybrid and content-based filtering techniques are often applied -- see e.g.~\cite{Adams2010,Fang2011, Menon2010, Shan2010,Zhou2012} for a non-exhaustive list. From a bioinformatics viewpoint, Settings B and C are often analyzed using graph-based methods that take the structure of a biological network into account -- see e.g.~\cite{Schrynemackers2013} for a recent review. When the features of $\bm{t}$ are negligible or unavailable, while those of $\bm{d}$ are informative, Setting B can be interpreted as a multi-label classification problem (in case of binary labels), a multi-output regression problem (in case of continuous labels) or, more generally, as a multi-task learning problem. Here, most techniques encode dependency in the tasks by means of a suitable loss function or by jointly regularizing the different tasks~ \citep{Dembczynski2012}. Prior knowledge on the tasks can be incorporated by using a feature description of the tasks that captures the relations between the tasks. Setting C is closer to transfer learning, in which one wants to generalize to new tasks. Here as well, a large number of applicable methods is available in the literature~\citep{Pan2010a}. 

Matrix factorization and hybrid filtering strategies are not applicable to Setting D. We will refer to this setting as the \emph{zero-shot learning} problem. This setting finds important applications in domains such as bioinformatics and chemistry -- see experiments. Compared to the other three settings, Setting D has received less attention in the literature (but it is gaining rapidly in popularity, see e.g. \cite{Larochelle2008,park2009pairwise,Menon2010,Palatucci2009,pahikkala2013conditional,Rohrbach2011a}). In the experimental section, we will investigate the transition phase between Settings C and D, when $\bm{t}$ occurs very few times in the training dataset, while $\bm{d}$ of the dyad $(\bm{d},\bm{t})$ is only observed in the prediction phase. We refer to this setting as an \emph{almost zero-shot learning} problem. 

Full and almost zero-shot learning problems can only be solved by considering feature representations of dyads. For Setting D Kronecker-based kernel methods are often employed~\citep{Vert2007,Brunner2012}. They have been successfully applied in order to solve problems such as product recommendation~\citep{Basilico2004, park2009pairwise}, enzyme annotation~\citep{Stock2014}, prediction of protein-protein~\citep{Benhur2005,Kashima2009} or protein-nucleic acid~\citep{Pelossof2015} interactions, drug design \citep{Jacob2008}, prediction of game outcomes \citep{pahikkala2010} and document retrieval \citep{pahikkala2013conditional}. For classification and regression problems, a standard recipe consists of plugging pairwise kernels in support vector machines, kernel ridge regression (KRR), or any other kernel method. Efficient optimization approaches based on gradient descent~\citep{park2009pairwise,Kashima2009} and closed-form solutions \citep{pahikkala2013conditional} have been proposed. In the next section, we will review those methods more in detail. 


\section{Pairwise learning using ridge-regression-based methods}\label{dyadicprediciton}


In this section an overview of kernel methods for pairwise learning is given, with a particular focus on the two-step method that will be further analyzed in upcoming sections. To this end, let us further extend the mathematical notation that was introduced before. Let $\objset=\{\bm{d}_i \mid i=1,\dots,\osize\}$ and $\taskset=\{\bm{t}_j \mid j=1,\ldots, \qsize\}$ denote, respectively, the sets of distinct instances and tasks encountered in the training set with $\osize=\arrowvert\objset\arrowvert$ and $\qsize=\arrowvert\taskset\arrowvert$. We say that the training set is \emph{complete} if it contains every instance-task pair with instance in $\objset$ and task in $\taskset$ exactly once. For complete training sets, we introduce a further notation for the matrix of labels $\bm{Y}\in\mathbb{R}^{\osize\times\qsize}$, so that its rows are indexed by the instances in $\objset$ and the columns by the tasks in $\taskset$.

The prediction function will be denoted as $\predfun(\bm{d},\bm{t})$, with $\predfun_j(\bm{d})$ the model for task $\bm{t}_j$. In some cases we will denote the learning algorithm in superscript, when this is omitted the model should be clear from the context or the statement is generally valid. For notational simplicity and generality, we will mostly use the dual form in this work. The matrix of parameters for the different models will be denoted as $\bm{A}=[a_{ij}]\in\mathbb{R}^{\osize\times\qsize}$, where each column corresponds to a parameter set for a different task. $\bm{A}\transpose$ denotes the transpose of $\bm{A}$. In some cases we need to work with the parameters transformed to a column vector: $\ve (\mathbf{A}) =\boldsymbol{\alpha}=[\alpha_k]\in\mathbb{R}^{\osize\qsize}$. Finally, let $\bm{B}_{i.}$ denote the $i$-th row and $\bm{B}_{.j}$ the $j$-th column of the matrix $\bm{B}$; we also use $\bm{B}_{ij}$ to refer to the element at the $i$-th row and $j$-th column of $\bm{B}$. 


\subsection{Independent-task kernel ridge regression}

Suppose that the training set is complete and that for each task we have $\osize$ labeled dyads $\objset=\{\bm{d}_i\}_{i=1}^{\osize}$. Let $\bm{Y}_{.j}\in\mathbb{R}^\osize$ be the labels of task $\bm{t}_j$ and $\kernelf(\cdot,\cdot)$ be a suitable kernel function which quantifies the similarity between the different instances. Since a separate and independent model is trained for each task, we will denote this setting as independent task (IT) kernel ridge regression. For each task $\bm{t}_j$, one would like to learn a function of the form
\eq{
\predfun_j^{\mathrm{IT}}(\bm{d}) = \sum_{i=1}^{\osize}
a_{ij}^\mathrm{IT} \kernelf(\bm{d},\bm{d}_i) \,,
}
with $a_{ij}^\mathrm{IT}$ parameters that minimize a suitable objective function. In the case of kernel ridge regression (KRR), this objective function is the squared loss with an $L_2$-complexity penalty. The parameters for the individual tasks using KRR can be found jointly by minimizing the following objective function~\citep{Bishop2006}:
\eqn{
J(\bm{A}^\mathrm{IT}) = \trace [{(\dkernelm\bm{A}^\mathrm{IT} - \bm{Y})\transpose(\dkernelm\bm{A}^\mathrm{IT} - \bm{Y})}] + \regparam_d \trace[{\bm{A}^\mathrm{IT}}\transpose\dkernelm\bm{A}^\mathrm{IT}]\, , \label{KRRpenalty}
}
with $\bm{A}^\mathrm{IT}=[a_{ij}^\mathrm{IT}]\in\mathbb{R}^{\osize\times\qsize}$ and $\dkernelm \in\mathbb{R}^{\osize\times\osize}$ the Gram matrix associated with the kernel function $\kernelf(\cdot,\cdot)$ for the instances. For simplicity, we assume the same regularization parameter $\regparam_d$ for each task, though extensions to different penalties for different tasks are straightforward. This basic setting assumes no crosstalk between the tasks as each model is fitted independently. The optimal coefficients that minimize Eq.\ (\ref{KRRpenalty}) can be found by solving the following linear system:
\eqn{\label{systemKRR}
\left(\dkernelm+\regparam_d\idmatrix\right)\bm{A}^\mathrm{IT}=\bm{Y}\,.
}
Using the singular value decomposition of the Gram matrix, this system can be solved for any value of $\regparam_d$ with a time complexity of $\mathcal{O}(\osize^3 +\osize^2 \qsize)$.

\subsection{Pairwise and Kronecker kernel ridge regression}

Suppose one has prior knowledge about which tasks are more similar, quantified by a kernel function $\gkernelf(\cdot, \cdot)$ defined over the tasks. Several authors (see \citet{Alvarez2012,Baldassarre2012} and references therein) have extended KRR to incorporate task correlations via matrix-valued kernels. However, most of this literature concerns kernels for which the tasks are fixed at training time. An alternative approach, allowing for the generalization to new tasks more straightforwardly by means of such a task kernel, is to use a pairwise kernel $\kkernelf\left(\left(\bm{d},\bm{t}\right),\left(\overline{\bm{d}},\overline{\bm{t}}\right)\right)$. Pairwise kernels provide a prediction function of type
\eqn{\label{tensorprediciton}
\predfun(\bm{d},\bm{t}) &= \sum_{h=1}^{\lsize}
\alpha_h\kkernelf\left(\left(\bm{d},\bm{t}\right),\left({\bm{d}}_h,{\bm{t}}_h\right)\right) \,,
}
where $\boldsymbol{\alpha}= [\alpha_h]$ are parameters that minimize the following objective function similar to the one used for independent task KRR:
\eqn{\label{tikhonov}
J(\bm{\boldsymbol\alpha})=(\bm{\kkernelm}\bm{\boldsymbol\alpha} -\bm{y})\transpose(\bm{\kkernelm}\bm{\boldsymbol\alpha} -\bm{y})+\regparam\bm{\boldsymbol\alpha }\transpose\kkernelm\bm{\boldsymbol\alpha} \,.
}
The minimizer can also be found by solving a system of linear equations:
\eqn{\label{kronsystem}
\left(\kkernelm+\regparam\idmatrix\right)\bm{\boldsymbol\alpha} =\bm{y}\,,
}
with $\kkernelm$ the Gram matrix. The most commonly used pairwise kernel is the Kronecker product pairwise kernel \citep{Basilico2004,oyama2004using,Benhur2005,park2009pairwise,Hayashi2012,Bonilla2007,pahikkala2013conditional}. This kernel is defined as
\eqn{\label{pairwisekernel}
\kkernelf^\mathrm{KK}\left(\left(\bm{d},\bm{t}\right),\left(\overline{\bm{d}},\overline{\bm{t}}\right)\right)=\kernelf\left(\bm{d},\overline{\bm{d}}\right)\gkernelf\left(\bm{t},\overline{\bm{t}}\right)
}
as a product of the data kernel $\kernelf(\cdot,\cdot)$ and the task kernel $\gkernelf(\cdot,\cdot)$. Many other variations of pairwise kernels have been considered to incorporate prior knowledge on the nature of the relations (e.g.~\cite{Vert2007,pahikkala2010,Waegeman2012,pahikkala2013conditional}) or for more efficient calculations in certain settings, e.g.~\citep{Kashima2010}.

Let $\bm{\tkernelm}\in\mathbb{R}^{\qsize\times\qsize}$ be the Gram matrix for the tasks. Then, for a complete training set, the Gram matrix for the instance-task pairs is the Kronecker product $\kkernelm=\bm{\tkernelm}\otimes\bm{\dkernelm}$. Often it is infeasible to use this kernel directly due to its large size. When the dataset is complete, the prediction function (Eq.~(\ref{tensorprediciton})) can be written as
\eqn{
\predfun^\mathrm{KK}(\bm{d}, \bm{t})=\sum_{i=1}^\osize \sum_{j=1}^\qsize a_{ij}^\mathrm{KK} \kernelf(\bm{d},\bm{d}_i) \gkernelf(\bm{t},\bm{t}_j)\,.\label{pairwisedual}
}
The matrix $\mathbf{F}$ containing the predictions for the training data using a pairwise kernel can be obtained by a linear transformation of the training labels:
\eqn{
\ve(\mathbf{F}) &= \kkernelm\left(\kkernelm+\regparam\idmatrix\right)^{-1}\ve(\bm{Y})\label{PAIRWISEREEST}\\
&=\bm{H}^\kkernelf\ve(\bm{Y})\label{hatk} \,,
}
with $\ve$ the vectorization operator which stacks the columns of a matrix into a vector. In the statistical literature $\bm{H}^\kkernelf=\kkernelm\left(\kkernelm+\regparam\idmatrix\right)^{-1}$ is denoted as the so-called hat matrix~\citep{hastie01statisticallearning}, which transforms the measurements into estimates.
As a special case of the Kronecker KRR, we also retrieve ordinary Kronecker kernel least-squares (OKKLS), when the objective function of Eq.~(\ref{tikhonov}) has no regularization term (i.e.~$\lambda=0$).

Several authors have pointed out that, while the size of the system in Eq.~(\ref{kronsystem}) is considerably large, its solutions for the Kronecker product kernel can be found efficiently via tensor algebraic optimization \citep{VanLoan2000,martin2006shiftedkron,Kashima2009,Raymond2010scalable,pahikkala2013conditional,Alvarez2012}. This is because the eigenvalue decomposition of a Kronecker product of two matrices can easily be computed from the eigenvalue decomposition of the individual matrices. The time complexity scales roughly with $O(\osize^3+\qsize^3)$, which is required for computing the singular value decomposition of both the instance and task kernel matrices, but the complexities can be scaled down even further by using sparse kernel matrix approximation.

However, these computational short-cuts only concern the case in which the training set is complete. If some of the instance-task pairs in the training set are missing or if there are several occurrences of certain pairs, one has to resort, for example, to gradient-descent-based training approaches~\citep{park2009pairwise,Kashima2009,pahikkala2013conditional}. While the training can be accelerated via tensor algebraic optimization, those techniques still remain considerably slower than the approach based on eigenvalue decomposition. 

\subsection{Two-step kernel ridge regression}

Clearly, independent-task ridge regression can generalize to new instances, but not to new tasks as no dependence between these tasks is encoded in the model. Kronecker KRR on the other hand can be used for all four prediction settings depicted in Figure~\ref{conceptrelations}. But since our definition of `instances' and `tasks' is purely conventional, nothing is stopping us from building a model using the kernel function $\gkernelf(\cdot,\cdot)$ to generalize to new tasks for the same instances. By combining two ordinary kernel ridge regressions, one for generalizing to new instances and one that generalizes for new tasks, one can indirectly predict for new dyads.

More formally, suppose one wants to make a prediction for the dyad $(\bm{d},\bm{t})$. Let $\bm{k}\in\mathbb{R}^{\osize}$ denote the vector of instance kernel evaluations between the instances in the training set and an instance in the test set, i.e.\ $\bm{k}(\bm{d})=\left(\kernelf(\bm{d},\bm{d}_1),\ldots,\kernelf(\bm{d},\bm{d}_\osize)\right)\transpose$. Likewise, $\bm{g}\in\mathbb{R}^{\qsize}$ represents the vector of task kernel evaluations between the target task and the auxiliary tasks, i.e.
$\bm{g}(\bm{t})=\left(\gkernelf(\bm{t},\bm{t}_1),\ldots,\gkernelf(\bm{t},\bm{t}_\qsize)\right)\transpose$. Based on the parameters found by solving Eq.~(\ref{systemKRR}), we can make a prediction for the new instance $\bm{d}$ for all the auxiliary tasks:

\eqn{\label{secondKRR}
\mathbf{\predfun}_\taskset (\bm{d}) = \bm{k}\transpose \left(\dkernelm+\regparam_d \idmatrix\right)^{-1} \bm{Y}  \,,
}
with $\regparam_d$ the specific regularization parameter for the instances. This vector of predictions $\mathbf{\predfun}_\taskset (\bm{d})$ can be used as a set of labels in an intermediate step to train a second model for generalizing to new tasks for the same instance. Thus, using the task kernel and a regularization parameter for the tasks $\regparam_t$, one obtains:
\eq{
\predfun^\mathrm{TS}(\bm{d}, \bm{t}) =  \bm{g}\transpose \left(\tkernelm+\regparam_t \idmatrix\right)^{-1}  \mathbf{\predfun}_\taskset (\bm{d})\transpose \,,
}
or, by making use of Eq.~(\ref{secondKRR}), the prediction is given by
\eqn{\label{2SRLS}
\predfun^\mathrm{TS}(\bm{d}, \bm{t}) &= \bm{k}\transpose \left(\dkernelm+\regparam_d \idmatrix\right)^{-1} \bm{Y}  \left(\tkernelm+\regparam_t \idmatrix\right)^{-1} \bm{g} \\
&= \bm{k}\transpose \bm{A}^\mathrm{TS} \bm{g}\,,
}
with $\bm{A}^\mathrm{TS}$ the dual parameters. We call this method two-step (TS) kernel ridge regression and it was independently proposed as embarrassingly simple zero-shot learning by~\cite{Romera-paredes2015}. It is represented in Figure~\ref{concept2SRLS}. Superficially, this approach resembles alternating least-squares~\citep{Zachariah2012a}, though the latter is an iteratively trained model mainly used for Setting A to obtain a low-rank representation. Our method on the other hand has a closed-form solution for the model parameters and allows for some computational techniques discussed later in this work. Two-step KRR can be used for any of the settings discussed in Section~\ref{related_settings}. Note that in practice there is no need to explicitly calculate $\mathbf{\predfun}_\taskset$, nor does it matter if in the first step one uses a model for new tasks and in the second step for instances, or the other way around. 

This model can be cast in a similar form as the pairwise prediction function of Eq.~(\ref{pairwisedual}) by making use of the identity $\ve(\bm{M}\bm{X}\bm{N}) = (\bm{N}\transpose\otimes\bm{M})\ve(\bm{X})$. Thus for two-step kernel ridge regression the parameters are given by

\eqn{
\bm{A}^\mathrm{TS} &=\left(\dkernelm+\regparam_d \idmatrix\right)^{-1} \bm{Y}  \left(\tkernelm+\regparam_t \idmatrix\right)^{-1}\,. \label{partwostepmatrix}
}

The time complexity for two-step kernel ridge regression is the same as for Kronecker KRR: $O(\osize^3+\qsize^3)$. The parameters can also be found by computing the eigenvalue decomposition of the two Gram matrices. Starting from these eigenvalue decompositions, it is possible to directly obtain the dual parameters for any values of the regularization hyperparameters $\lambda_d$ and $\lambda_t$. Because of its conceptual simplicity, it is quite straightforward to use two-step KRR for certain cases when the label matrix is not complete, in contrast to the Kronecker KRR, see experimental section. The computational advantages of this method will be discussed in Section~\ref{computationalaspects}. Table~\ref{pairwisemodels} gives an overview of the different learning methods considered in this section.

\begin{figure}[t]
   \begin{center}
   \includegraphics[width=0.7\textwidth]{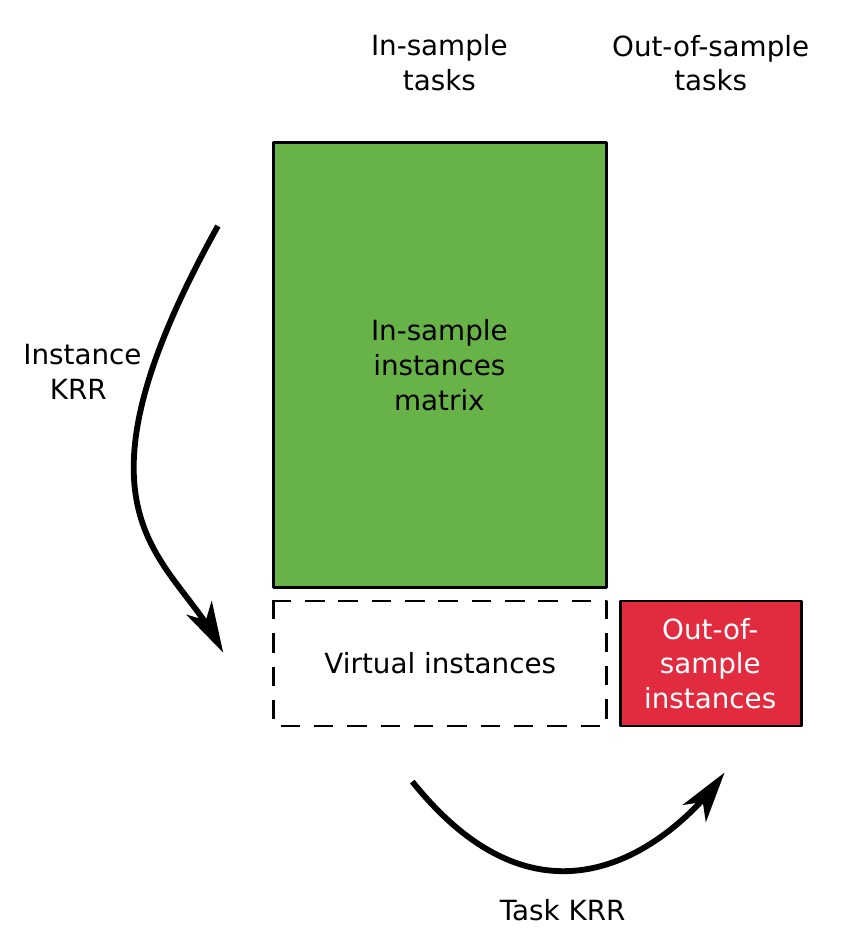} 
   \end{center}
   \caption{Principle of two-step kernel ridge regression. In a first step, a virtual prediction is made for the auxiliary tasks for new instances using a first KRR model. A second KRR model is trained using this data and this model is used to make predictions for new tasks.}
   \label{concept2SRLS}
\end{figure}

\begin{table}[t]
  \begin{center}
   \topcaption{Overview of the different learning methods discussed in this work. For each method, the dual form of the prediction function and the linear system for estimating the parameters is given. Here, the data is assumed to be complete.} 
   \label{pairwisemodels} 
   \begin{tabular}{ l  ll} 
\hline\hline
KRR method & Dual prediction function & System for dual parameters \\
\hline
Independent-task & $\displaystyle\predfun_j^{\mathrm{IT}}(\bm{d}) = \sum_{i=1}^{\osize}
a_{ij}^\mathrm{IT} \kernelf(\bm{d},\bm{d}_i)$ & $\displaystyle\left(\dkernelm+\regparam_d\idmatrix\right)\bm{A}^\mathrm{IT}=\bm{Y}$\\
Kronecker kernel & $\displaystyle\predfun^\mathrm{KK}(\bm{d}, \bm{t})=\sum_{i=1}^\osize \sum_{j=1}^\qsize a_{ij}^\mathrm{KK} \kernelf(\bm{d},\bm{d}_i) \gkernelf(\bm{t},\bm{t}_j)$ & $\displaystyle\left(\bm{\tkernelm}\otimes\bm{\dkernelm}+\regparam\idmatrix\right)\ve(\bm{A}^\mathrm{KK})=\ve(\bm{Y})$\\
\begin{tabular}[x]{@{}c@{}}Ordinary Kronecker \\kernel least-squares \end{tabular}& $\displaystyle\begin{array} {r@{}l@{}} \predfun^\mathrm{OKKLS} & {}  =\\ & \displaystyle{} \sum_{i=1}^\osize \sum_{j=1}^\qsize a_{ij}^\mathrm{OKKLS} \kernelf(\bm{d},\bm{d}_i) \gkernelf(\bm{t},\bm{t}_j) \end{array}$  & $\displaystyle(\bm{\tkernelm}\otimes\bm{\dkernelm})\ve(\bm{A}^\mathrm{OKKLS})=\ve(\bm{Y})$\\
Two-step & $\displaystyle\predfun^\mathrm{TS}(\bm{d}, \bm{t})=\sum_{i=1}^\osize \sum_{j=1}^\qsize a_{ij}^\mathrm{TS} \kernelf(\bm{d},\bm{d}_i) \gkernelf(\bm{t},\bm{t}_j)$ & $\displaystyle\left(\dkernelm+\regparam_d \idmatrix\right)\bm{A}^\mathrm{TS}\left(\tkernelm+\regparam_t \idmatrix\right) =\bm{Y}$\\
\hline
\hline
   \end{tabular}
   \end{center}
   \end{table}

\section{Theoretical considerations}\label{theoreticalconsiderations}
In this section we will show that two-step kernel ridge regression can be seen as using kernel ridge regression with special kinds of pairwise kernel matrices, depending on the prediction setting. We will show that Setting A is a transductive setting while Setting D is merely a special case of (Kronecker) kernel ridge regression. We will also study the different learning algorithms from a spectral filtering point of view, showing that two-step kernel ridge regression uses a special decomposable filter. From these observations we will prove the universality and admissibility of the methods. 

\subsection{Equivalence between two-step and other kernel ridge regression methods}

The relation between two-step ridge regression and independent-task ridge regression is given in the following theorem. 

\begin{mytheo}[Setting B]\label{ITKRRTSeqiuivalence}
When the Gram matrix of the tasks $\bm{\tkernelm}$ is full rank and $\lambda_t$ is set to zero, independent-task KRR and two-step KRR return the same predictions for any given training task:
\eq{
\predfun^\mathrm{IT}_j(\cdot)\equiv\predfun^\mathrm{TS}(\cdot, \bm{t}_j)\,.
}
\end{mytheo}

\begin{proof}
The prediction for the independent-task KRR is given by:
\eq{
\predfun^\mathrm{IT}_j(\bm{d}) = [\bm{k}\transpose(\dkernelm+\regparam_d\idmatrix)^{-1} \bm{Y}]_j \, .
}
For two-step KKR, it follows from Eq.~(\ref{2SRLS}) that
\eq{
\predfun^\mathrm{TS}_j(\bm{d}) &= [\bm{k}\transpose(\dkernelm+\regparam_d\idmatrix)^{-1} \bm{Y}\tkernelm^{-1}\tkernelm]_j \\
 &=  [\bm{k}\transpose(\dkernelm+\regparam_d\idmatrix)^{-1} \bm{Y}]_j\,.
}\end{proof}
When $\tkernelm$ is singular, the $\qsize$ outputs for the different tasks are projected on a lower-dimensional subspace by two-step KRR. This means that a dependence between the tasks is enforced, even with $\lambda_t=0$.

The connection between two-step and Kronecker kernel ridge regression is established by the following results. 
\begin{mytheo}[Setting A]\label{KKTSequivalenceA}
Consider the following pairwise kernel matrix:
\eq{
\bm{\Xi} &= \bm{G}\otimes\bm{K}\left(\lambda_d\lambda_t\bm{I}\otimes\bm{I}+\lambda_t\bm{I}\otimes\bm{K}+\lambda_d\bm{G}\otimes\bm{I}\right)^{-1}\,.}
The predictions for the training data $\bm{F}$ using pairwise KRR (Eq.\ (\ref{PAIRWISEREEST})) with the above pairwise kernel and regularization parameter $\lambda=1$ correspond to those obtained with two-step KRR using the kernel matrices $\bm{K}$, $\bm{G}$ with respective regularization parameters $\lambda_d$ and $\lambda_t$.
\end{mytheo}
\begin{proof}

We will write the corresponding empirical risk minimization of Eq.\ (\ref{tikhonov}) from the perspective of value regularization. Since Setting A is an imputation setting, we directly search for the optimal predicted label matrix $\bm{F}$, rather than the optimal parameter matrix. Starting from the objective function for KRR, the predictions for the training data are obtained via minimizing the following variational problem:
\eqn{
J(\bm{F})&=\ve(\bm{F}-\bm{Y})\transpose\ve(\bm{F}-\bm{Y})+\ve(\bm{F})\transpose\bm{\Xi}^{-1}\ve(\bm{F}) \label{eq:TSERM}\\
&=\ve(\bm{F}-\bm{Y})\transpose\ve(\bm{F}-\bm{Y}) \nonumber \\
&\qquad+\ve(\bm{F})\transpose\left(\bm{G}\otimes\bm{K}\left(\lambda_d\lambda_t\bm{I}\otimes\bm{I}+\lambda_t\bm{I}\otimes\bm{K}+\lambda_d\bm{G}\otimes\bm{I}\right)^{-1}\right)^{-1}\ve(\bm{F})\nonumber \\ 
&=\ve(\bm{F}-\bm{Y})\transpose\ve(\bm{F}-\bm{Y}) \nonumber \\
&\qquad+\ve(\bm{F})\transpose\left(\bm{G}^{-1}\otimes\bm{K}^{-1}\left(\lambda_d\lambda_t\bm{I}\otimes\bm{I}+\lambda_t\bm{I}\otimes\bm{K}+\lambda_d\bm{G}\otimes\bm{I}\right)\right)\ve(\bm{F})\nonumber \\ 
&=\ve(\bm{F}-\bm{Y})\transpose\ve(\bm{F}-\bm{Y}) \nonumber \\
&\qquad+\ve(\bm{F})\transpose\left(\lambda_d\lambda_t\bm{G}^{-1}\otimes\bm{K}^{-1}+\lambda_d\bm{I}\otimes\bm{K}^{-1}+\lambda_t\bm{G}^{-1}\otimes\bm{I}\right)\ve(\bm{F})\nonumber \\ 
&=\trace((\bm{F}-\bm{Y})\transpose(\bm{F}-\bm{Y})
+\lambda_d\lambda_t\bm{F}\transpose\bm{K}^{-1}\bm{F}\bm{G}^{-1}+\lambda_d\bm{F}\transpose\bm{K}^{-1}\bm{F}+\lambda_t\bm{F}\transpose\bm{F}\bm{G}^{-1})\,.\nonumber
}
The derivative with respect to $\bm{F}$ is:
\eq{
\frac{\partial J}{\partial \bm{F}}&=2(\bm{F}-\bm{Y}+\lambda_d\lambda_t\bm{K}^{-1}\bm{F}\bm{G}^{-1}+\lambda_d\bm{K}^{-1}\bm{F}+\lambda_t\bm{F}\bm{G}^{-1})\\
&=2(\lambda_d\bm{K}^{-1}+\bm{I})\bm{F}(\lambda_t\bm{G}^{-1}+\bm{I})-2\bm{Y}\,.
}
Setting it to zero and solving with respect to $\bm{F}$ yields:
\eq{
\bm{F}&=(\lambda_d\bm{K}^{-1}+\bm{I})^{-1}\bm{Y}(\lambda_t\bm{G}^{-1}+\bm{I})^{-1}\\
&=\bm{K}(\bm{K}+\lambda_d\bm{I})^{-1}\bm{Y}(\bm{G}+\lambda_t\bm{I})^{-1}\bm{G}\,.
}
Comparing with Eq.~(\ref{partwostepmatrix}), we note that $\bm{F}=\bm{K}\bm{A}^\mathrm{TS}\bm{G}$, which proves the theorem.
\end{proof}

Here, we have assumed that $\dkernelm$ and $\tkernelm$ are invertible. The kernel $\bm{\Xi}$ can always be obtained as long as $\dkernelm$ and $\tkernelm$ are positive semi-definite. The relevance of the above theorem is that it formulates two-step KRR as an empirical risk minimization problem for Setting~A (Eq.~(\ref{eq:TSERM})). It is important to note that the pairwise kernel matrix $\bm{\Xi}$ only appears in the regularization term of this variational problem. The loss function is only dependent on the predicted values $\bm{F}$ and the label matrix $\bm{Y}$. Using two-step KRR for Setting A when dealing with incomplete data is thus well defined. The empirical risk minimization problem of Eq.~(\ref{eq:TSERM}) can be modified so that the squared loss only takes the observed dyads into account:  
\eq{
J(\bm{F})=\sum_{(\mathbf{d}, \mathbf{t}, y)\in S}(y - \predfun(\mathbf{d}, \mathbf{t})))^2+\ve(\bm{F})\transpose\bm{\Xi}^{-1}\ve(\bm{F})\,,
}
with $S$ the training set of labeled dyads. In this case, one ends up with a transductive setting. This explains why Setting A is the most easy setting to predict for, as will be shown in the experiments. See \cite{Rifkin2007,Johnson2008} for a more in-depth discussion. 

Two-step and Kronecker KRR also coincide in an interesting way for zero-shot learning problems (e.g.\ the special case in which there is no labeled data available for the target task). This, in turn, allows us to show the consistency of two-step KRR via its universal approximation and spectral regularization properties. The theorem below shows the relation between two-step KRR and ordinary Kronecker kernel ridge regression for Setting D.
\begin{mytheo}[Setting D] \label{KKTSequivalence}
Let us consider a zero-shot learning setting with a complete training set. Let $\predfun^\mathrm{TS}(\cdot, \cdot)$ be a model trained with two-step KRR and $\predfun^\mathrm{OKKLS}(\cdot, \cdot)$ be a model trained with ordinary Kronecker kernel least-squares regression (OKKLS) using the following pairwise kernel function on $\objspace\times\taskspace$:
\eqn{\label{twostepkernel}
\Upsilon\left(\left(\bm{d},\bm{t}),(\overline{\bm{d}},\overline{\bm{t}}\right)\right)
=\left(\kernelf\left(\bm{d},\overline{\bm{d}}\right)
+\regparam_d\delta\left(\bm{d},\overline{\bm{d}}\right)\right)
\left(\gkernelf\left(\bm{t},\overline{\bm{t}})
+\regparam_t\delta\left(\bm{t},\overline{\bm{t}}\right)\right)\right)
}
where $\delta$ is the delta kernel whose value is 1 if the arguments are equal and 0 otherwise.
Then for making predictions for instances $\bm{d}\in\objspace\setminus\objset$ and tasks $\bm{t}\in\taskspace\setminus\taskset$ not seen in the training set, it holds that $\predfun^\mathrm{TS}(\bm{t},\bm{d})=\predfun^\mathrm{OKKLS}(\bm{t}, \bm{d})$.
\end{mytheo}

\begin{proof}
From Eq.~(\ref{2SRLS}) we have the following dual model for prediction:
\eq{
\predfun^\mathrm{TS}(\bm{d}, \bm{t})=\sum_{i=1}^\osize \sum_{j=1}^\qsize a_{ij}^\mathrm{TS} \kernelf(\bm{d},\bm{d}_i) \gkernelf(\bm{t},\bm{t}_j)\,,
}
with $\bm{A}^\mathrm{TS}=[a_{ij}^\mathrm{TS}]$ the matrix of parameters. Similarly, the dual representation of the OKKLS (see Eq.~(\ref{pairwisedual})), using a parametrization $\bm{A}^\mathrm{OKKLS}=[a_{ij}^\mathrm{OKKLS}]$, is given by
\eq{
\predfun^\mathrm{OKKLS}(\bm{d}, \bm{t})&=\sum_{i=1}^\osize \sum_{j=1}^\qsize a_{ij}^\mathrm{OKKLS} \Upsilon\left(\left(\bm{d},\bm{t}),(\bm{d}_i,\bm{t}_j\right)\right)\,, \\
&=\sum_{i=1}^\osize \sum_{j=1}^\qsize a_{ij}^\mathrm{OKKLS}\left(\kernelf\left(\bm{d},{\bm{d}}_i\right)
+\regparam_d\delta\left(\bm{d},{\bm{d}}_i\right)\right)
\left(\gkernelf\left(\bm{t},{\bm{t}}_j)
+\regparam_t\delta\left(\bm{t},{\bm{t}}_j\right)\right)\right)\,,\\
&=\sum_{i=1}^\osize \sum_{j=1}^\qsize a_{ij}^\mathrm{OKKLS}\kernelf(\bm{d},\bm{d}_i) \gkernelf(\bm{t},\bm{t}_j)\,.\\
}
In the last step we used the fact that $\bm{d}\neq\bm{d}_i$ and $\bm{t}\neq\bm{t}_j$ to drop the delta kernels. Hence, we need to show that $\bm{A}^\mathrm{TS}=\bm{A}^\mathrm{OKKLS}$.

By Eq.~(\ref{partwostepmatrix}) and denoting $\widetilde{\bm{G}}=\left(\bm{G}+\regparam\idmatrix\right)^{-1}$ and $\widetilde{\bm{K}}=\left(\bm{K}+\regparam\idmatrix\right)^{-1}$, we observe that the model parameters $\bm{A}^\mathrm{TS}$ of two-step model can also be obtained from the following closed form:
\eqn{
\bm{A}^\mathrm{TS}&=\widetilde{\bm{K}}
\bm{Y}\widetilde{\bm{G}}
\,.
}
The kernel matrix of $\Upsilon$ for Setting D can be expressed as:
$
\bm{\Upsilon} = \left(\bm{G}+\regparam_t\idmatrix\right)\otimes\left(\bm{K}+\regparam_d\idmatrix\right)\,.
$
The OKKLS problem with kernel $\Upsilon$ being
\[
\ve(\bm{A}^\mathrm{OKKLS})=\argmin_{\bm{A}^\mathrm{OKKLS}\in\mathbb{R}^{\osize\times\qsize}}\left(\ve(\bm{Y})-\bm{\Upsilon}\ve(\bm{A}^\mathrm{OKKLS})\right)\transpose\left(\ve(\bm{Y})-\bm{\Upsilon}\ve(\bm{A}^\mathrm{OKKLS})\right)\,,
\]
its minimizer can be expressed as
\eqn{\label{OKKLSparameters}
\ve(\bm{A}^\mathrm{OKKLS})&=\bm{\Upsilon}^{-1}\ve(\bm{Y})
=\left(\left(\bm{G}+\regparam_t\idmatrix\right)^{-1}\otimes\left(\bm{K}+\regparam_d\idmatrix\right)^{-1}\right)\ve(\bm{Y})\nonumber\\
&=\ve\left(\left(\bm{K}+\regparam_d\idmatrix\right)^{-1}\bm{Y}\left(\bm{G}+\regparam_t\idmatrix\right)^{-1}\right)
=\ve\left(\widetilde{\bm{K}}\bm{Y}\widetilde{\bm{G}}\right)\,. 
}
Here, we again make use of the identity $\ve(\bm{M}\bm{X}\bm{N}) = (\bm{N}\transpose\otimes\bm{M})\ve(\bm{X})$, which holds for any conformable matrices $\bm{M}$, $\bm{X}$, and $\bm{N}$. From Eq.~(\ref{OKKLSparameters}) it then follows that $\bm{A}^\mathrm{TS}=\bm{A}^\mathrm{OKKLS}$, which proves the theorem.
\end{proof}

\subsection{Universality of the Kronecker product pairwise kernel}
Here we consider the universal approximation properties of two-step kernel ridge regression. This is a necessary step in showing the consistency of this method. We first recall the concept of universal kernel functions.
\begin{definition} \label{kerneluniversalitydef}
{\bf \citep{Steinwart2002consistency}} A continuous kernel $\kernelf(\cdot,\cdot)$ on a compact metric space $\anyspace$ (i.e.~$\anyspace$ is closed and bounded) is called universal if the reproducing kernel Hilbert space (RKHS) induced by $\kernelf(\cdot,\cdot)$ is dense in $C(\anyspace)$, where $C(\anyspace)$ is the space of all continuous functions $\predfun : \anyspace \rightarrow \mathbb{R}$.
\end{definition}
The universality property indicates that the hypothesis space induced by a universal kernel can approximate any continuous function on the input space $\anyspace$ to be learned arbitrarily well, given that the available set of training data is large and representative enough, and the learning algorithm can efficiently find this approximation from the hypothesis space \citep{Steinwart2002consistency}. In other words, the learning algorithm is consistent in the sense that, informally put, the hypothesis learned by it gets closer to the function to be learned while the size of the training set gets larger. The consistency properties of two-step KRR are considered in more detail in Section~\ref{spectralInterpretation}.

\newcommand{\tempfirstfun}{u}
\newcommand{\tempsecondfun}{v}
\newcommand{\temppairfun}{t}

Next, we consider the universality of the Kronecker product pairwise kernel. The following result is a straightforward modification of some of the existing results in the literature (e.g. \citet{Waegeman2012}), but we present it here for self-sufficiency. This theorem is mainly related to Setting D, while it also covers the other settings as special cases.
\begin{mytheo}
The kernel $\kkernelf^\mathrm{KK}((\cdot,\cdot),(\cdot,\cdot))$ on $\objspace\times\taskspace$ defined in Eq.\ (\ref{pairwisekernel}) is universal if the instance kernel $\kernelf(\cdot,\cdot)$ on $\objspace$ and the task kernel $\gkernelf(\cdot,\cdot)$ on $\taskspace$ are both universal.
\end{mytheo}
\begin{proof}
Let us define
\begin{equation}\label{funckron}
\begin{array}{l}
\mathcal{A}\otimes\mathcal{B}
=\left\{\temppairfun\mid t(\bm{d},\bm{t})=\tempfirstfun(\bm{d})v(\bm{t}),\tempfirstfun\in \mathcal{A},\tempsecondfun\in \mathcal{B}\right\}
\end{array}
\end{equation}
for compact metric spaces $\mathcal{D}$ and $\mathcal{T}$ and sets of functions $\mathcal{A}\subset C(\mathcal{D})$ and $\mathcal{B}\subset C(\mathcal{T})$. We observe that the RKHS of the kernel $\Gamma$ can be written as $\mathcal{H}(\kernelf)\otimes\mathcal{H}(\gkernelf)$, where $\mathcal{H}(\kernelf)$ and $\mathcal{H}(\gkernelf)$ are the RKHS of the kernels $\kernelf(\cdot,\cdot)$ and $\gkernelf(\cdot,\cdot)$, respectively.

Let $\epsilon>0$ and let $\temppairfun\in C(\mathcal{D})\otimes C(\mathcal{T})$ be an arbitrary function which can, according to Eq.~(\ref{funckron}), be written  as $\temppairfun(\bm{d},\bm{t})=\tempfirstfun(\bm{d})\tempsecondfun(\bm{t})$, where $\tempfirstfun\in C(\mathcal{D})$ and $\tempsecondfun\in C(\mathcal{T})$. By definition of the universality property, $\mathcal{H}(\kernelf)$ and $\mathcal{H}(\gkernelf)$ are dense in $C(\mathcal{D})$ and $C(\mathcal{T})$, respectively. Therefore, there exist functions $\overline{\tempfirstfun}\in\mathcal{H}(\kernelf)$ and $\overline{\tempsecondfun}\in\mathcal{H}(\gkernelf)$ such that
\[
\max_{\bm{d}\in\mathcal{D}}\left\arrowvert \overline{\tempfirstfun}(\bm{d})-\tempfirstfun(\bm{d})\right\arrowvert\leq \overline{\epsilon},\qquad \max_{\bm{t}\in\mathcal{T}}\left\arrowvert \overline{\tempsecondfun}(\bm{t})-\tempsecondfun(\bm{t})\right\arrowvert\leq \overline{\epsilon} \,,
\]
where $\overline{\epsilon}$ is a constant for which it holds that
\[
\max_{\bm{d}\in\mathcal{D},\bm{t}\in\mathcal{T}}\left\{\left\arrowvert \overline{\epsilon} \, \tempfirstfun(\bm{d})\right\arrowvert+\left\arrowvert\overline{\epsilon} \,\tempsecondfun(\bm{t})\right\arrowvert+\overline{\epsilon}^2\right\}\leq \epsilon \,.
\]
Note that, according to the extreme value theorem, the maximum exists due to the compactness of $\mathcal{D}$ and $\mathcal{T}$ and the continuity of the functions $\tempfirstfun(\cdot)$ and $\tempsecondfun(\cdot)$. Now we have
\[
\begin{array}{l}
\displaystyle
\max_{\bm{d}\in\mathcal{D},\bm{t}\in\mathcal{T}}\left\arrowvert t(\bm{d},\bm{t})-\overline{\tempfirstfun}(\bm{d})\overline{\tempsecondfun}(\bm{t})\right\arrowvert\\
\displaystyle
\leq\max_{\bm{d}\in\mathcal{D},\bm{t}\in\mathcal{T}}\left\{\left\arrowvert t(\bm{d},\bm{t})-\tempfirstfun(\bm{d})\tempsecondfun(\bm{t})\right\arrowvert+\left\arrowvert \overline{\epsilon} \,\tempfirstfun(\bm{d})\right\arrowvert+\left\arrowvert\overline{\epsilon}\,\tempsecondfun(\bm{t})\right\arrowvert+\overline{\epsilon}^2\right\}\\
\displaystyle
=\max_{\bm{d}\in\mathcal{D},\bm{t}\in\mathcal{T}}\left\{\left\arrowvert \overline{\epsilon} \,\tempfirstfun(\bm{d})\right\arrowvert+\left\arrowvert\overline{\epsilon} \,\tempsecondfun(\bm{t})\right\arrowvert+\overline{\epsilon}^2\right\}\leq \epsilon,
\end{array}
\]
which confirms the density of $\mathcal{H}(\kernelf)\otimes\mathcal{H}(\gkernelf)$ in $C(\mathcal{D})\otimes C(\mathcal{T})$.

The space $\mathcal{D}\times\mathcal{T}$ is compact if both $\mathcal{D}$ and $\mathcal{T}$ are compact according to Tikhonov's theorem. It is straightforward to see that $C(\mathcal{D})\otimes C(\mathcal{T})$ is a subalgebra of $C(\mathcal{D}\times\mathcal{T})$, it separates points in $\mathcal{D}\times\mathcal{T}$, it vanishes at no point of $C(\mathcal{D}\times\mathcal{T})$, and it is therefore dense in $C(\mathcal{D}\times\mathcal{T})$ due to the Stone-Weierstra{\ss} theorem. Thus, $\mathcal{H}(\kernelf)\otimes\mathcal{H}(\gkernelf)$ is also dense in $C(\mathcal{D}\times\mathcal{T})$, and $\Gamma$ is a universal kernel on $\mathcal{D}\times\mathcal{T}$.
\end{proof}

\subsection{Spectral interpretation}\label{spectralInterpretation}
In Theorem~\ref{KKTSequivalence} we have shown the relation between two-step kernel ridge regression and (Kronecker) kernel ridge regression for Setting D. In this section we will study the difference between single-task, Kronecker and two-step kernel ridge regression from the point of view of spectral regularization. The above shown universal approximation properties of this kernel are also connected to the consistency properties of two-step KRR, as is elaborated in more detail in this section.

Learning by spectral regularization originates from the theory of ill-posed problems. This paradigm is well studied in domains such as image analysis~\citep{Bertero1998} and, more recently, in machine learning -- e.g.~\cite{LoGerfo2008}. Here, one wants to find the parameters $\bm{\boldsymbol\alpha}$ of the data-generating process given a set of noisy measurements $\labelvec$ such that
\eqn{\label{illposedproblem}
\kkernelm\bm{\boldsymbol\alpha} \approx \labelvec\,,
}
with $\kkernelm$ a Gram matrix with eigenvalue decomposition $\kkernelm = \bm{W}\bm{\Lambda}\bm{W}\transpose$. At first glance, one can find the parameters $\bm{\boldsymbol\alpha}$ by inverting $\kkernelm$:
\eq{
\bm{\boldsymbol\alpha} & = \kkernelm^{-1}\labelvec \\
 &=  \bm{W}\bm{\Lambda}^{-1}\bm{W}\transpose\labelvec\,.
}
If $\kkernelm$ has small eigenvalues, the inverse becomes highly unstable: small changes in the feature description of the label vector will lead to huge changes in $\bm{\boldsymbol\alpha}$. Spectral regularization deals with this problem by generalizing the inverse by a so-called filter function to make solving Eq.~(\ref{illposedproblem}) well-posed. The following definition of a spectral filter-based regularizer is standard in the machine learning literature (see e.g. \citet{LoGerfo2008} and references therein).
Note that we assume $\kkernelf((\cdot,\cdot),(\cdot,\cdot))$ being bounded with $\kappa>0$ such that $\sup_{\bm{x}\in\ispace}\sqrt{\kkernelf(\bm{x},\bm{x})}\leq\kappa$, assuring that the eigenvalues of the Gram matrix $\kkernelm$ are in $[0,\kappa^2]$.
\begin{definition}[Admissible regularizer]\label{filterdef}
A function $\filterfun_\regparam:[0,\kappa^2]\rightarrow \mathbb{R}$, parameterized by $0<\regparam\leq\kappa^2$, is an admissible regularizer if there exist constants $D, B,\gamma\in\mathbb{R}$ and $\bar{\nu},\gamma_\nu > 0$ such that
\[
\sup_{0<\sigma\leq\kappa^2}\arrowvert\sigma\filterfun_\regparam(\sigma)\arrowvert\leq D\textnormal{, }
\sup_{0<\sigma\leq\kappa^2}\arrowvert\filterfun_\regparam(\sigma)\arrowvert\leq\frac{B}{\regparam}\textnormal{, }
\sup_{0<\sigma\leq\kappa^2}\arrowvert 1-\sigma\filterfun_\regparam(\sigma)\arrowvert\leq \gamma\,,
\]
\[
\text{and }\sup_{0<\sigma\leq\kappa^2} \frac{\regparam^\nu}{\sigma^\nu}\arrowvert 1-\sigma\filterfun_\regparam(\sigma)\arrowvert\leq \gamma_\nu,\, \quad\text{for any }\nu\in\,]0,\bar{\nu}]\,,
\]
where the constant $\gamma_\nu$ does not depend on $\regparam$.
\end{definition}
The constant $\bar{\nu}$ is in the literature called the qualification of the regularizer and it is related to the consistency properties of the learning method as described in more detail below.

The spectral filter is a matrix function that acts as a stabilized generalization of a matrix inverse. Hence, Eq.~(\ref{illposedproblem}) can be solved by
\eq{
\bm{\boldsymbol\alpha}&=\filterfun_\regparam(\kkernelm)\labelvec\\
&=\bm{W}\filterfun_\regparam(\bm{\Lambda})\bm{W}\transpose\ve(\bm{Y})\,.
} 
Similarly, the noisy measurements can be filtered to obtain a better estimation of the true labels:
\eq{
\bm{\predfun} &= \kkernelm\bm{\boldsymbol\alpha}\nonumber\\
&= \bm{W}\bm{\Lambda}\bm{W}\transpose\bm{W}\filterfun_\regparam(\bm{\Lambda})\bm{W}\transpose\ve(\bm{Y}) \nonumber \\
&= \bm{W}\bm{\Lambda}\filterfun_\regparam(\bm{\Lambda})\bm{W}\transpose\ve(\bm{Y})\,. \nonumber
}
The spectral interpretation allows for using a more general form of the hat matrix (Eq.~(\ref{hatk})):
\eq{\bm{H}^\kkernelf=\bm{W}\bm{\Lambda}\filterfun_\regparam(\bm{\Lambda})\bm{W}\transpose\,.
}
For example, the filter function corresponding to the Tikhonov regularization, as used for independent-task kernel ridge regression, is given by
\eq{
\filterfun^\mathrm{TIK}_\regparam(\bm{\sigma})=\frac{1}{\bm{\sigma}+\regparam}\,,
}
with the ordinary least-squares approach corresponding to $\regparam=0$. Several other learning approaches, such as spectral cut-off, iterated Tikhonov and $L2$ Boosting, can also be expressed as filter functions, but cannot be expressed as a penalized empirical error minimization problem analogous to Eq.~(\ref{tikhonov})~\citep{LoGerfo2008}. The spectral interpretation can also be used to motivate novel learning algorithms.
 
Many authors have expanded this framework to multi-task settings, e.g.~\cite{Argyriou2007,Argyriou2010,Baldassarre2012}. We will translate the pairwise learning methods from Section~\ref{dyadicprediciton} to this spectral regularization context. Let us denote the eigenvalue decomposition of the instance and task kernel matrices as
\eq{
\bm{\dkernelm} = \bm{U} \bm{\Sigma} \bm{U}\transpose \qquad \bm{\tkernelm} = \bm{V} \bm{S} \bm{V}\transpose\,.
}
Let $\bm{u}_i$ denote the $i$-th eigenvector of $\bm{\dkernelm}$ and $\bm{v}_j$ the $j$-th eigenvector of $\bm{\tkernelm}$. The eigenvalues of the kernel matrix obtained with the Kronecker product kernel on a complete training set can be expressed as the Kronecker product $\bm{\Lambda}=\bm{S}\otimes\bm{\Sigma}$ of the eigenvalues $\bm{\Sigma}$ and $\bm{S}$ of the instance and task kernel matrices. For the models of this work, it is opportune to define a pairwise filter function over the representation of the instances and tasks.

\newcommand{\ebound}{a}
\newcommand{\rbound}{b}
We also note that we assume both of the factor kernels to be bounded, and hence we can write that all the eigenvalues $\varsigma$ of the Kronecker product kernel can be factorized as the product of the eigenvalues of the instance and task kernels as follows:
\eqn{\label{evalfactorization}
\varsigma=\sigma s\qquad\textnormal{with }0\leq\sigma, s\leq\ebound\sqrt{\varsigma}\textnormal{ and }1\leq\ebound<\infty\,,
}
where $\sigma,s$ denote the eigenvalues of the factor kernels and $\ebound$ the constant determined as the product of $\sup_{\bm{d}\in\objspace}\sqrt{\kernelf(\bm{d},\bm{d})}$ and  $\sup_{\bm{t}\in\taskspace}\sqrt{\gkernelf(\bm{t},\bm{t})}$.
\begin{definition}[Pairwise spectral filter]
We say that a function $\filterfun_\regparam:[0,\kappa^2]\rightarrow \mathbb{R}$, parameterized by $0<\regparam\leq\kappa^2$, is a pairwise spectral filter if it can be written as
\[
\filterfun_\regparam(\varsigma)=\vartheta_\regparam(\sigma,s)
\]
for some function $\vartheta_{\regparam}:[0,\ebound\sqrt{\varsigma}]^2\rightarrow \mathbb{R}$ with $1\leq\ebound<\infty$, and it is an admissible regularizer for all possible factorizations of the eigenvalues as in Eq.~(\ref{evalfactorization}).
\end{definition}

Since the eigenvalues of a Kronecker product of two matrices are just the scalar product of the eigenvalues of the matrices, the filter function for Kronecker KRR is given by 
\eqn{\label{kronfilter}
\vartheta_\regparam^\mathrm{KK}(s, \sigma) = \filterfun_\lambda^\mathrm{TIK}(\sigma s)=\frac{1}{(\sigma s+\regparam)} \,,
}
where $\sigma$ and $s$ are the eigenvalues of $\bm{K}$ and $\bm{G}$, respectively. The admissibility of this filter is a well-known result, since it is simply the Tikhonov regularizer for the pairwise Kronecker product kernel. 

Instead of considering two-step kernel ridge regression from the kernel point of view, one can also cast it into the spectral filtering regularization framework. We start from Eq.~(\ref{partwostepmatrix}) in vectorized form:
\eq{
\ve(\bm{A}) &= \left((\bm{\tkernelm} + \lambda_t\idmatrix)^{-1} \otimes (\bm{\dkernelm} + \lambda_d\idmatrix)^{-1}\right) \ve(\bm{Y})\\
 &= \left(( \bm{V} \bm{S} \bm{V}\transpose+ \lambda_t\idmatrix)^{-1} \otimes (\bm{U} \bm{\Sigma} \bm{U}\transpose+ \lambda_d\idmatrix)^{-1}\right) \ve(\bm{Y})\\
& =\left(( \bm{V} \filterfun^\mathrm{TIK}_{\regparam_t}(\bm{S}) \bm{V}\transpose) \otimes (\bm{U} \filterfun^\mathrm{TIK}_{\regparam_d}(\bm{\Sigma})  \bm{U}\transpose)\right) \ve(\bm{Y})\\
& =\left((  \bm{V}\otimes \bm{U}) (\filterfun^\mathrm{TIK}_{\regparam_t}(\bm{S})\otimes \filterfun^\mathrm{TIK}_{\regparam_d}(\bm{\Sigma}) ) (  \bm{V}\otimes \bm{U})\transpose\right) \ve(\bm{Y})\,.
}
Hence, one can interpret two-step KRR with a complete training set for Setting D as a spectral filtering regularization based learning algorithm that uses the pairwise Kronecker product kernel with the following filter function.
\eqn{
\vartheta^\mathrm{TS}_{\regparam}(s,\sigma)&=\filterfun^\mathrm{TIK}_{\regparam_t}(s)\filterfun^\mathrm{TIK}_{\regparam_d}(\sigma)\nonumber\\
& =\frac{1}{(\sigma+\regparam_d)(s+\regparam_t)}\nonumber\\
&=\frac{1}{\sigma s+\regparam_t\sigma+\regparam_ds+\regparam_t\regparam_d}\,.\label{twostepfilter}
}
The validity of this filter is characterized by the following theorem.
\begin{mytheo}
The filter function $\vartheta^\mathrm{TS}_{\regparam}(\cdot,\cdot)$ is admissible with $D=B=\gamma=1$, $\gamma_\nu=2\ebound\rbound$, and has qualification $\bar{\nu}=\frac{1}{2}$ for all factorizations of $\varsigma$ and $\regparam$ as
\eqn{\label{factorizations}
\varsigma=\sigma s\textnormal{ and }\regparam=\regparam_t\regparam_d\quad\text{ with }0\leq\sigma, s\leq\ebound\sqrt{\varsigma}\textnormal{ and }0<\regparam_t,\regparam_d\leq\rbound\sqrt{\regparam}\,,
}
where $1\leq\ebound,\rbound<\infty$ are constants that do not depend on $\regparam$ or $\varsigma$.
\end{mytheo}
\begin{proof}
Let us recollect the last condition in Definition~\ref{filterdef}:
\[
\sup_{0<\varsigma\leq\kappa^2}\frac{\varsigma^\nu}{\regparam^\nu}\arrowvert 1-\varsigma\filterfun_\regparam(\varsigma)\arrowvert\leq \gamma_\nu,\quad\text{for any }\nu\in\,]0,\bar{\nu}]\,,
\]
where $\gamma_\nu$ does not depend on $\regparam$. In order to show this for all cases covered by Eq.~(\ref{factorizations}), we rewrite the condition by taking the supremum with respect to the factorizations of $\varsigma$ and $\regparam$ given the constants $\ebound$ and $\rbound$:
\eq{
\sup_{
\underset{\underset{ 0<\sigma, s\leq\ebound\sqrt{\varsigma}}{ 0<\regparam_t,\regparam_d\leq\rbound\sqrt{\regparam}}}{0<\varsigma\leq\kappa^2}}\frac{\varsigma^\nu}{\regparam^\nu}\left(1- \frac{\varsigma}{\varsigma+\regparam_t\sigma+\regparam_d s+\regparam}\right)\leq \gamma_\nu,\qquad\text{ for any }\nu\in\,]0,\bar{\nu}]\,.
}
The left-hand side then becomes
\eq{
\sup_{0<\varsigma\leq\kappa^2}\frac{\varsigma^\nu}{\regparam^\nu}\left(1- \frac{\varsigma}{\varsigma+2\ebound\rbound\sqrt{\regparam}\sqrt{\varsigma}+\regparam}\right)
=\sup_{0<\varsigma\leq\kappa^2}\left(\frac{2\ebound\rbound\regparam^{\frac{1}{2}-\nu}\varsigma^{\nu+\frac{1}{2}}+\regparam^{1-\nu}\varsigma^\nu}{\varsigma+2\ebound\rbound\sqrt{\regparam}\sqrt{\varsigma}+\regparam}\right)\,.
}
By checking the extreme values of the latter expression with respect to $(\varsigma,\regparam,\nu)$ using standard differential calculus, we observe that it is bounded by $\gamma_\nu=2\ebound\rbound$ if $\nu\in\,]0,\frac{1}{2}]$. With values of $\bar{\nu}$ larger than $\frac{1}{2}$, the term $2\ebound\rbound\regparam^{\frac{1}{2}-\nu}\varsigma^{\nu+\frac{1}{2}}$ in the numerator grows arbitrarily while $\regparam\rightarrow 0$, and hence the qualification is $\bar{\nu}=\frac{1}{2}$. The other conditions in Definition~\ref{filterdef} can be checked by direct computation.
\end{proof}
Thus, Eq.~(\ref{twostepfilter}) can be positioned within the spectral filtering regularization framework with separate regularization parameter values for instances and tasks. In contrast to Eq.~(\ref{kronfilter}), the filter of two-step KRR can be factorized into a component for the tasks and instances separately:
 \eqn{
\vartheta_\regparam(s, \sigma)=\filterfun_{\regparam_d}(\sigma)\filterfun_{\regparam_t}(s)\label{decompfilter}\,.
}
This decomposition gives rise to some computational shortcuts for performing cross-validation, as will be discussed in Section~\ref{computationalaspects}.

Providing a different regularization for instances and tasks also makes sense from a learning point of view. It is easy to imagine a setting in which the instance has a much larger influence in determining the label compared to the task or vice versa. For example, consider a collaborative filtering setting with the goal of recommending books for customers. Suppose that the sales of a book is for a very large part determined simply by being a bestseller novel or not, and less by individual customers' taste. When building a predictive model, one would give more freedom to the part concerning the books (hence a lower regularization). Less degrees of freedom are given to the inference of the users' personal task, as this is harder to learn and explains less of the variance in the preferences. This can be extended even further, by choosing specific filter functions separately for the instances and tasks tuned to the application at hand. In a pairwise setting, the filter function to perform independent-task KRR arises as a special case with $\lambda_t=0$:
\eq{
\vartheta^\mathrm{IT}_{\regparam_t}(s,\sigma) = \frac{1}{(\sigma+\lambda_d)s}\,,
}
when the task kernel is full rank (see Theorem~\ref{ITKRRTSeqiuivalence}).

Next, we analyze the consistency properties of two-step KRR in setting D, given the above results about the universality of the pairwise Kronecker product kernel and the spectral filtering interpretation of the method.
Let $R$ denote the expected prediction error of a hypothesis $\predfun$ with respect to some unknown probability measure $\rho(\bm{x}, y)$ on the joint space $\ispace\times\mathbb{R}$ of inputs and labels, that is,
\[
R(\predfun) = \int_{\ispace\times\mathbb{R}}(\predfun(\bm{x})- y)^2 d\rho(\bm{x}, y)\,.
\]
Given the input space $\ispace$, the minimizer of the error is the so-called regression function:
\[
\predfun_\rho(\bm{x})=\int_{\mathbb{R}} y\ d\rho(y\mid\bm{x})\,.
\]
Following \citet{Bauer2007regularization,LoGerfo2008,Baldassarre2012}, we characterize the quality of a learning algorithm via its consistency properties. In particular, by considering whether the learning algorithm is consistent in the following sense:
\begin{definition}
\label{consistencydef}
A learning algorithm is consistent if the following holds with high probability
\[
\lim_{\lsize\rightarrow\infty}
\int_{\ispace}\left(\hat{\predfun}^\regparam_\lsize(\bm{x}) - \predfun_\rho(\bm{x})\right)^2 d\rho(\bm{x})=0\,,
\]
where $\hat{\predfun}^\regparam_\lsize$ denotes the hypothesis inferred by the learning algorithm from a training set having $\lsize$ independently and identically drawn training examples.
\end{definition}

The following result is assembled from the existing literature concerning spectral filtering based regularization methods and we present it here only in a rather abstract form. For the exact details and further elaboration, we refer to \citet{Bauer2007regularization,LoGerfo2008,Baldassarre2012}. 
\begin{mytheo}
If the filter function is admissible and the kernel function is universal, then the learning algorithm is consistent in the sense of Def.~\ref{consistencydef}. Furthermore, if the regularization parameter is set as $\regparam=\frac{1}{\lsize^{2\bar{\nu}+1}}$, where $\lsize$ denotes the number of independently and identically drawn training examples, then the following holds with high probability:
\eqn{
R(\hat{\predfun}^\regparam)-R(\predfun_\rho(\bm{x}))
 = \mathcal{O}\left(\lsize^{-\frac{\bar{\nu}}{2\bar{\nu}+1}}\right)\,.
}
\end{mytheo}
Intuitively put, the universality of the kernel ensures that the regression function belongs to the hypothesis space of the learning algorithm
and the admissibility of the regularizer ensures that $R(\hat{\predfun}^\regparam)$ converges to it when the size of the training set approaches infinity and the rate of convergence is reasonable.

\begin{corollary}
Two-step KRR is consistent and the hypothesis it infers from the training set of size $\lsize=\osize\qsize$ converges to the underlying regression function with a rate at least proportional to
\eqn{
R(\hat{\predfun}^\regparam)-R(\predfun_\rho(\bm{d}, \bm{t}))
 = \mathcal{O}\left(\min(\osize,\qsize)^{-\frac{\bar{\nu}}{2\bar{\nu}+1}}\right)\,.
}
\end{corollary}
\begin{proof}
The result follows from the admissibility of the pairwise filter function, the universality of the pairwise Kronecker product kernel and the fact that the training set consists of at least $\min(\osize,\qsize)$ independently and identically drawn training examples.
\end{proof}

\section{Efficient algorithms for two-step kernel ridge regression}\label{computationalaspects}
In this section we derive some computational shortcuts for two-step kernel ridge regression. For Kronecker kernel ridge regression, it is well known that the huge system of Eq.~(\ref{kronsystem}) can be solved efficiently because the Gram matrix can be decomposed. Our two-step method takes this decomposition one step further, as the model can be seen as applying two consecutive regression steps. This allows to derive efficient algorithms for cross-validation for each of the four settings depicted in Figure~\ref{conceptrelations}, while the original Kronecker kernel ridge regression only allowed for a shortcut for Setting A. The same linear algebra can also be used to implement a scheme for online updating of the model with new instances or tasks. These shortcuts for cross-validating for Settings B, C and D and online updating cannot be derived for Kronecker ridge regression in general.

\subsection{Efficient hold-out computation}\label{HOOshortcuts}
As indicated in Figure~\ref{conceptrelations}, making predictions in a dyadic setting is much more complex compared to the classical case when there is only one task. This implies that the correct way to assess the performance and to do model selection is also more complicated. To this end we suggest cross-validation settings which take the structure of the label matrix into account. Depending on the prediction setting of interest, one should withhold individual elements, rows, columns or both of the label matrix, as shown in Figure~\ref{CVillustration}. These schemes have been discussed in other works, often in the context of predicting interactions in molecular biology, e.g.~\citep{Park2012,Pahikkala2015}. In this work, we will only focus on deriving shortcuts for leave-one-out cross-validation for these settings. Extensions to general hold-out schemes can be obtained using similar reasonings as the one in this section. 

It is well known that for independent-task kernel ridge regression one can efficiently compute the values for leave-one-out cross-validation, provided one has stored the hat matrix~\citep{rifkin2007notes}. For the instances, using the kernel $\kernelf(\cdot,\cdot)$, the hat matrix is denoted as $\bm{H}^\kernelf = \dkernelm(\dkernelm +\lambda_d\idmatrix)^{-1}$ (see Eq.~(\ref{hatk})). As noted earlier, if this matrix is obtained using an eigenvalue decomposition of $\dkernelm$, $\bm{H}^\kernelf$ can be computed efficiently for any $\lambda_d$:
\eq{
\bm{H}^\kernelf & = \dkernelm(\dkernelm +\lambda_d\idmatrix)^{-1} \\
 & = \bm{U} \bm{\Sigma} \bm{U}\transpose (\bm{U} \bm{\Sigma} \bm{U}\transpose +\lambda_d\idmatrix )^{-1} \\
 & = \bm{U} \bm{\Sigma} \bm{U}\transpose \bm{U} (\bm{\Sigma}  +\lambda_d\idmatrix )^{-1} \bm{U}\transpose \\
 & =  \bm{U}  (\bm{\Sigma} (\bm{\Sigma} + \lambda_d \idmatrix)^{-1} ) \bm{U}\transpose \,.
}
\begin{figure}[t]
   \centering
   \includegraphics[width=0.6\textwidth]{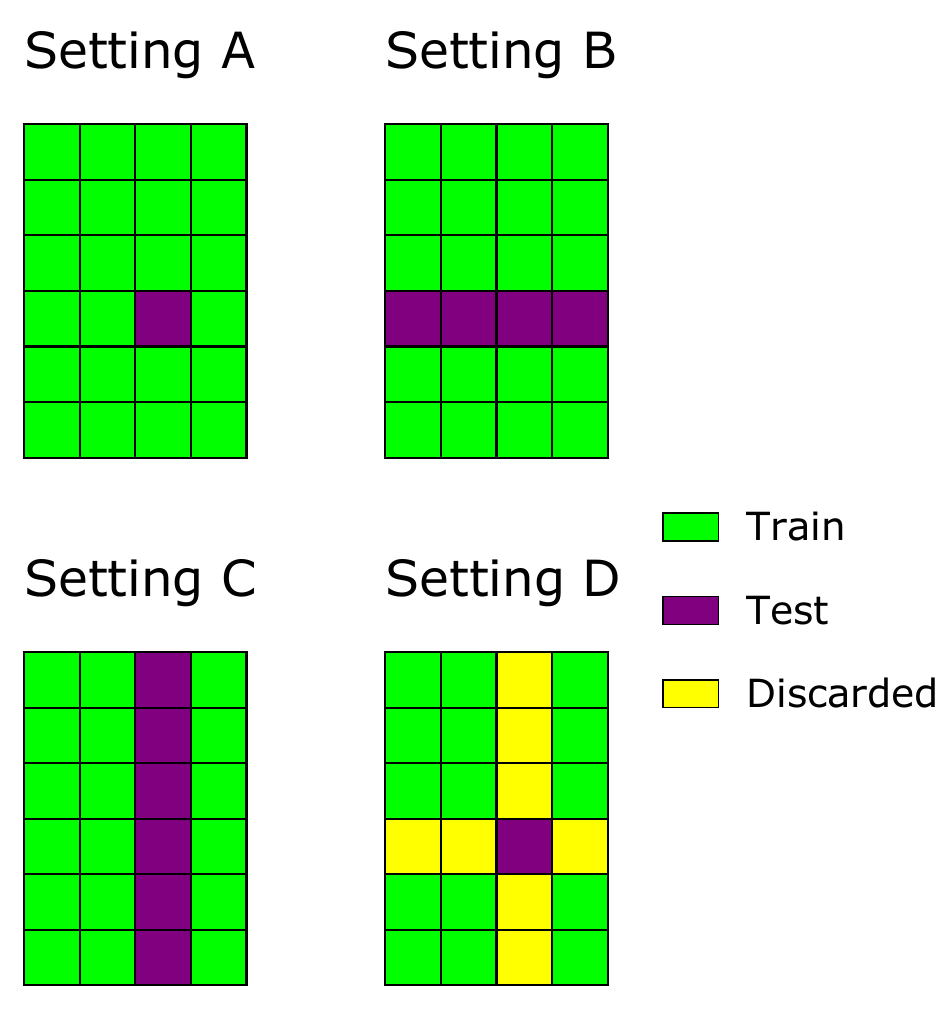} 
   \caption{Overview of different versions of leave-one-out cross-validation settings discussed in Section~\ref{related_settings} applied to a problem with six instances and four tasks.}
   \label{CVillustration}
\end{figure}

First, we will derive the well-known shortcut for independent-task kernel ridge regression for calculating the leave-one-out cross-validation values, using our notation. We start with the classical leave-one-out shortcut.
\begin{mytheo}\label{loocvProp}
A single row of the matrix $\mathbf{F}^\mathrm{IT, LOO}$ containing the labels re-estimated by leave-one-out cross-validation using independent-task kernel ridge regression with associated hat matrix $\bm{H}^\kernelf$, can be calculated as
\eq{
\mathbf{F}^{\mathrm{IT, LOO}}_{i.} = \frac{\bm{H}^\kernelf_{i.}\bm{Y} - \bm{H}^\kernelf_{ii}\bm{Y}_{i.} }{1-\bm{H}^\kernelf_{ii}}\,.
}
\end{mytheo}
\begin{proof}
This is merely a multivariate version of the leave-one-out shortcut, proven in texts such as~\cite{Wahba1990,Pahikkala2006,rifkin2007notes}.
\end{proof}
The above theorem can be applied to ridge regression and related models. The shortcut is relevant for both Settings A and B, as it is used to estimate how well the model can generalize to new instances, though kernel ridge regression does not use any information on the tasks. Starting from this general shortcut and using the connection between independent-task, two-step and Kronecker KRR, we can derive shortcuts for Setting A.
\begin{mycol}[Setting A]\label{loocvKKA}
A single element of the matrix containing the labels re-estimated by leave-one-out cross-validation for Setting A, i.e.\ leaving out one dyad at a time, using Kronecker kernel ridge regression, can be calculated as
\eq{
\mathbf{F}^\mathrm{KK, LOO, A}_{ij} = \frac{\bm{H}^\mathrm{KK}_{s.}\ve (\bm{Y}) -\bm{H}_{ss}^\mathrm{KK}\bm{Y}_{ij}}{1 - \bm{H}^\mathrm{KK}_{ss}}\,,
}
with $\bm{H}^\mathrm{KK} = (\tkernelm\otimes\dkernelm)(\tkernelm\otimes\dkernelm + \lambda\idmatrix)^{-1}$ and $s =\osize j+i$.
\end{mycol}
\begin{proof}
We have noted earlier that Kronecker kernel ridge regression is merely kernel ridge regression using the Gram matrix $\tkernelm\otimes\dkernelm$ and $\ve (\bm{Y})$ as the single output label vector. Since $\bm{H}^\mathrm{KK}$ is the corresponding hat matrix, the proof follows directly from rephrasing Theorem~\ref{loocvProp} in this terminology.
\end{proof}
By using the eigenvalue decomposition of the kernel matrices, the diagonal elements of $\bm{H}^\mathrm{KK}$ and $\bm{H}^\mathrm{KK}\ve(\bm{Y})$ can be computed efficiently.

\begin{mycol}[Setting A]\label{loocvTSA}
A single element of the matrix containing the labels re-estimated by leave-one-out cross-validation for Setting A, i.e.\ leaving out one dyad at a time, using two-step kernel ridge regression can be calculated as
\eq{
\mathbf{F}^\mathrm{TS, LOO, A}_{ij} = \frac{\bm{H}^\mathrm{TS}_{s.}\ve (\bm{Y}) -\bm{H}_{ss}^\mathrm{TS}\bm{Y}_{ij}}{1 - \bm{H}^\mathrm{TS}_{ss}}\,,
}
with $\bm{H}^\mathrm{TS} = (\tkernelm+\lambda_t\idmatrix)^{-1}\tkernelm\otimes(\dkernelm+\lambda_d\idmatrix)^{-1}\dkernelm$ and $s =\osize j+i$.
\end{mycol}
\begin{proof}
Theorem~\ref{KKTSequivalenceA} states that two-step KRR is merely kernel ridge regression with the Gram matrix $\bm{\Xi}$ and $\ve (\bm{Y})$ as the single output label vector. From Eq.~(\ref{partwostepmatrix}) it follows that $\bm{H}^\mathrm{TS}$ is the hat matrix for two-step kernel ridge regression in vectorized form:
\eq{
\bm{f} &= \ve (\dkernelm\bm{A}^\mathrm{TS}\tkernelm)\\
&= \ve (\dkernelm\left(\dkernelm+\regparam_d \idmatrix\right)^{-1} \bm{Y}  \left(\tkernelm+\regparam_t \idmatrix\right)^{-1}\tkernelm)\\
&=[(\tkernelm+\lambda_t\idmatrix)^{-1}\tkernelm\otimes(\dkernelm+\lambda_d\idmatrix)^{-1}\dkernelm]\ve(\bm{Y})\,.
}
Hence, the proof follows directly from rephrasing Theorem~\ref{loocvProp} in this terminology.
\end{proof}
These two shortcuts are of interest when validating or tuning models for collaborative filtering or network inference. For example, recently a model for recipe completion was validated by iteratively withholding single elements from the label matrix~\citep{DeClercq2015a}. An interesting application would be in supervised biological network inference, e.g. predicting interactions between biomolecules, such as proteins, ligands or DNA or predicting interactions between species, such as plants and pollinators (see e.g.~\cite{Rafferty2013,Hadfield2014} for applications of pairwise learning in such a setting). Such datasets are often plagued with false negatives or false positives, which make them difficult to analyze, see~\cite{Schrynemackers2013} for a discussion. Using the provided holdout tricks, it is possible to re-estimate each interaction extremely efficiently in order to screen for mislabeled observations.

Because two-step kernel ridge regression can be decomposed in two `steps', it is possible to derive shortcuts for leave-one-out cross-validation for Settings B, C and D. Below is a shortcut for Setting B, which allows for cross-validation of one row (instance) at a time.
\begin{mycol}[Setting B]\label{loocvTSB}
A single row of the matrix containing the labels re-estimated by leave-one-out cross-validation for Setting B, i.e.~leaving out one instance at a time, using two-step kernel ridge regression, can be calculated as
\eq{
\mathbf{F}^\mathrm{TS, LOO, B}_{i.} = \frac{\left(\mathbf{H}^\kernelf_{i.}\mathbf{Y}- \mathbf{H}^\kernelf_{ii}\mathbf{Y}_{i.}\right)\mathbf{H}^\gkernelf}{1- \mathbf{H}^\kernelf_{ii}}\,.
}
\end{mycol}

\begin{proof}
Using two-step KRR for Setting B can be thought of as just using independent-task KRR with $\mathbf{H}^\kernelf$ as the hat matrix and $ \bm{Y}\mathbf{H}^\gkernelf$ as the label matrix. The proof follows directly from rephrasing Theorem~\ref{loocvProp} in this terminology.
%
\end{proof}
Similarly, using the following shortcut one can perform cross-validation for Setting C, leaving out one column (task) at a time. 

\begin{mycol}[Setting C]\label{loocvTSC}
A single column of the matrix containing the labels re-estimated by leave-one-out cross-validation for Setting C, i.e.~leaving out one task at a time, using two-step kernel ridge regression, can be calculated as
\eq{
\mathbf{F}^\mathrm{TS, LOO, C}_{.j} = \frac{\mathbf{H}^\kernelf\left(\mathbf{Y}\mathbf{H}^\gkernelf_{.j}- \mathbf{Y}_{.j}\mathbf{H}^\gkernelf_{jj}\right)}{1 - \mathbf{H}^\gkernelf_{jj}}\,.
}
\end{mycol}
\begin{proof}
Similarly as for Corollary~\ref{loocvTSB} but using a `transposed' dataset, two-step KRR for Setting C can be thought of as just performing independent-task KRR with $\mathbf{H}^\gkernelf$ as the hat matrix and $ \bm{Y}\transpose\mathbf{H}^\kernelf$ as the label matrix. We apply Theorem~\ref{loocvProp} using this terminology and transpose the obtained matrix.
\end{proof}
Finally, the theorem below gives the shortcut for Setting D, the zero-shot learning setting. 
\begin{mycol}[Setting D]\label{loocvTSD}
A single element of the matrix containing the labels re-estimated by leave-one-out cross-validation for Setting D, i.e.\ leaving out one dyad with the corresponding row and column and predicting for that dyad, using two-step kernel ridge regression can be calculated as
\eq{
\mathbf{F}^\mathrm{TS, LOO, D}_{ij} = \frac{\mathbf{H}^\kernelf_{i.}\left(\mathbf{Y}\mathbf{H}^\gkernelf_{.j}- \mathbf{Y}_{.j}\mathbf{H}^\gkernelf_{jj}\right)
- \mathbf{H}^\kernelf_{ii}\left(\mathbf{Y}_{i.}\mathbf{H}^\gkernelf_{.j}- \mathbf{Y}_{ij}\mathbf{H}^\gkernelf_{jj}\right)
}{ 
(1-\bm{H}_{ii}^\kernelf)(1-\bm{H}_{jj}^\gkernelf)}\,.
}
\end{mycol}
\begin{proof}
This can easily be proven by applying Theorem~\ref{loocvProp} twice. First, a column vector that predicts for task $j$ is generated, using a model trained only based on the $\qsize-1$ remaining tasks. Adopting a similar reasoning as in Corollary~\ref{loocvTSC}, we obtain
\eq{
\mathbf{F}^\mathrm{STEP 1}_{.j} = \frac{\mathbf{Y}\mathbf{H}^\gkernelf_{.j}- \mathbf{Y}_{.j}\mathbf{H}^\gkernelf_{jj}}{1 - \mathbf{H}^\gkernelf_{jj}}\,.
} 
This vector is subsequently used as a label vector for an `unseen' task $j$. We apply Theorem~\ref{loocvProp} once more to obtain the leave-one-out value of instance (row) $i$:
\eq{
\mathbf{F}^\mathrm{TS, LOO, D}_{ij} &=  \frac{\bm{H}^\kernelf_{i.}\mathbf{F}^\mathrm{STEP 1}_{.j} - \bm{H}^\kernelf_{ii}\mathbf{F}^\mathrm{STEP 1}_{ij} }{1-\bm{H}^\kernelf_{ii}}\\
& = \frac{\bm{H}^\kernelf_{i.}\left(\frac{\mathbf{Y}\mathbf{H}^\gkernelf_{.j}- \mathbf{Y}_{.j}\mathbf{H}^\gkernelf_{jj}}{1 - \mathbf{H}^\gkernelf_{jj}}\right) - \bm{H}^\kernelf_{ii}\left(\frac{\mathbf{Y}_{i.}\mathbf{H}^\gkernelf_{.j}- \mathbf{Y}_{ij}\mathbf{H}^\gkernelf_{jj}}{1 - \mathbf{H}^\gkernelf_{jj}}\right) }{1-\bm{H}^\kernelf_{ii}}\\
&=\frac{\mathbf{H}^\kernelf_{i.}\left(\mathbf{Y}\mathbf{H}^\gkernelf_{.j}- \mathbf{Y}_{.j}\mathbf{H}^\gkernelf_{jj}\right)
- \mathbf{H}^\kernelf_{ii}\left(\mathbf{Y}_{i.}\mathbf{H}^\gkernelf_{.j}- \mathbf{Y}_{ij}\mathbf{H}^\gkernelf_{jj}\right)
}{ 
(1-\bm{H}_{ii}^\kernelf)(1-\bm{H}_{jj}^\gkernelf)}\,.
}\end{proof}

\begin{table}[t]
  \begin{center}
   \topcaption{Overview of algebraic shortcuts for performing leave-one-out prediction for the different settings using two-step kernel ridge regression. $\mat_{\osize\times\qsize}$ is an operator that reshapes a vector in an $\osize\times\qsize$ matrix, $\diagv$ takes the diagonal elements of a square matrix to a vector and $\diagm$ takes the corresponding diagonal matrix of a given square matrix and $\mathbf{1}_n$ is a vector of ones with a length of $n$.} 
   \label{looshortcuts} 
   \begin{tabular}{ c c c} 
\hline\hline
model & setting & leave-one-out values \\
\hline
IT & A \& B & $\displaystyle\frac{\bm{H}^\kernelf\bm{Y}-\diagm(\bm{H}^\kernelf)\bm{Y}}{\diagv(\idmatrix-\bm{H}^\kernelf)\bm{1}_\qsize\transpose}$\\[0.2cm]
KK & A & $\displaystyle\mat_{\osize\times\qsize}\left(\frac{\bm{H}^\mathrm{KK}\ve (\bm{Y}) -\diagm(\bm{H}^\mathrm{KK})\ve (\bm{Y})}{\diagv(\idmatrix - \bm{H}^\mathrm{KK})}\right)$\\[0.35cm]
TS & A & $\displaystyle\mat_{\osize\times\qsize}\left(\frac{\bm{H}^\mathrm{TS}\ve (\bm{Y}) -\diagm(\bm{H}^\mathrm{TS})\ve (\bm{Y})}{\diagv(\idmatrix - \bm{H}^\mathrm{TS})}\right)$\\[0.35cm]
TS & B & $\displaystyle\frac{\mathbf{H}^\kernelf\mathbf{Y}\mathbf{H}^\gkernelf - \diagm (\mathbf{H}^\kernelf )\mathbf{Y}\mathbf{H}^\gkernelf}{\diagv (\idmatrix - \mathbf{H}^\kernelf) \bm{1}_\qsize\transpose}$\\[0.35cm]
TS & C & $\displaystyle\frac{\mathbf{H}^\kernelf \mathbf{Y}\mathbf{H}^\gkernelf- \mathbf{H}^\kernelf \mathbf{Y}\diagm (\mathbf{H}^\gkernelf)}{ \bm{1}_\osize\transpose \diagv (\idmatrix - \mathbf{H}^\gkernelf)\transpose}$\\[0.35cm]
TS & D & $\displaystyle\frac{(\mathbf{H}^\kernelf - \diagm (\mathbf{H}^\kernelf )) \mathbf{Y}(\mathbf{H}^\gkernelf- \diagm (\mathbf{H}^\gkernelf))}{ \diagv (\idmatrix - \mathbf{H}^\kernelf) \diagv (\idmatrix - \mathbf{H}^\gkernelf)\transpose}$\\[0.35cm]
\hline
\hline
   \end{tabular}
   \end{center}
   
\end{table}

All the derived shortcuts are summarized in a more compact matrix notation in Table~\ref{looshortcuts}. Each matrix of hold-out values is computed by dividing the predicted labels of the complete training set by an appropriate linear transformation of the diagonal elements of the hat matrices. As noted earlier, using an eigenvalue decomposition with a time complexity of $\mathcal{O}(\osize^3+\qsize^3)$, these hat matrices and predictions can be obtained for any values of $\lambda_d$ and $\lambda_t$ by straightforward matrix manipulations. The shortcuts are in fact valid for all models with a pairwise spectral filter of the form in Eq.~(\ref{decompfilter}). Stated more boldly, after the eigenvalue decomposition, one can tune or validate the model for any of the four settings at essentially no computational cost, compared to the cost of the initial preprocessing. For Kronecker kernel ridge regression, the only computational shortcut of Table~\ref{looshortcuts} possible is for Setting A, as this makes no use of this property. Since many interesting learning problems relate to the other settings, we consider the shortcuts possible for two-step kernel ridge regression a very strong merit of this method.

\subsection{Learning with mini batches}\label{twostepminibatches}

In some cases it is desirable to train a model in an online fashion, as opposed to the standard batch training. For example, when new instances for the different tasks become available, it is more desirable to update the model parameters rather than completely retrain the model. Similarly, new tasks could be added to the model, which also influences the performance of the older tasks. This is prevalent in online applications, such as recommender systems, in which data is added dynamically so that the systems should be updated swiftly. In the context of dealing with large-scale data, the required matrices might simply be too large to fit in main memory as a whole, thus an iterative approach is needed.

In addition to gradient-based methods such as stochastic gradient descent and conjugated gradient descent, closed-from solutions can be derived for updating the model parameters for two-step KRR. Since we assume that this is of particular interest for large-scale data applications, we will derive a shortcut for the primal form. Denote a feature vector for the instances by $\boldsymbol{\phi}\in\mathbb{R}^\odsize$ and for the tasks by $\boldsymbol{\psi}\in\mathbb{R}^\qdsize$. Here we assume that $\odsize\ll\osize$ and $\qdsize\ll \qsize$. These can either be the primal features or be obtained from a decomposition of the kernel matrices, for example by means of the Nystr{\"o}m method~\citep{Drineas2005}. This leads to the associated feature matrices $\bm{\Phi}\in\mathbb{R}^{\osize\times\odsize}$ and $\bm{\Psi}\in\mathbb{R}^{\qsize\times\qdsize}$. Hence, the primal notation boils down to:
\eq{
\dkernelm &= \bm{\Phi}\bm{\Phi}\transpose \qquad \bm{k} = \bm{\Phi}\boldsymbol{\phi}\\
\tkernelm &= \bm{\Psi}\bm{\Psi}\transpose \qquad \bm{g} = \bm{\Psi}\boldsymbol{\psi}\,.
}
Thus the model of Eq.~(\ref{2SRLS}) translates to
\eq{
\predfun^\mathrm{TS}(\bm{x}) &= \boldsymbol{\phi}\transpose \bm{\Phi}\bm{A}^\mathrm{TS} \bm{\Psi}\transpose\boldsymbol{\psi}\\
&= \boldsymbol{\phi} \transpose \bm{W}^\mathrm{TS} \boldsymbol{\psi} \,,
}
with $\bm{W}^\mathrm{TS}$ the primal parameter matrix. Suppose that the dataset is again complete, then, starting from Eq.~(\ref{partwostepmatrix}), the parameters are given by
\eq{
\bm{W}^\mathrm{TS} &= \bm{\Phi}(\bm{\Phi}\bm{\Phi}\transpose + \lambda_d\idmatrix)^{-1}\bm{Y}(\bm{\Psi}\bm{\Psi}\transpose + \lambda_t\idmatrix)^{-1}\bm{\Psi}\transpose\\
& = (\bm{\Phi}\transpose\bm{\Phi} + \lambda_d\idmatrix)^{-1}\bm{\Phi}\transpose\bm{Y}\bm{\Psi}(\bm{\Psi}\transpose\bm{\Psi} + \lambda_t\idmatrix)^{-1} \,,
}
where we made use of the matrix inversion lemma~(\cite{Bishop2006}, Eq.~(C.7)). Suppose that in a first training phase, we use only the first $\osize-l$ instances to train the first model and later update the model with the remaining $l$ instances. Without loss of generality, we divide the labels and instance features as follows in corresponding block matrices
\eq{
\bm{\Phi} = 
\begin{pmatrix}
\bm{\Phi}_1\\
\bm{\Phi}_2
\end{pmatrix}  
\quad \mathrm{and}\quad
\bm{Y} = 
\begin{pmatrix}
\bm{Y}_1\\
\bm{Y}_2
\end{pmatrix}\,,
}
with $\bm{\Phi}_1\in\mathbb{R}^{(\osize-l)\times\odsize}$, $\bm{\Phi}_2\in\mathbb{R}^{l\times\odsize}$, $\bm{Y}_1\in\mathbb{R}^{(\osize-l)\times\qsize}$ and $\bm{Y}_2\in\mathbb{R}^{l\times\qsize}$. The model parameters of the complete dataset can then be calculated as
\eqn{
\bm{W}^\mathrm{TS} &= (\bm{\Phi}_1\transpose\bm{\Phi}_1 + \bm{\Phi}_2\transpose\bm{\Phi}_2 + \lambda_d\idmatrix)^{-1}(\bm{\Phi}_1\transpose\bm{Y}_1 + \bm{\Phi}_2\transpose\bm{Y}_2)\bm{\Psi}(\bm{\Psi}\transpose\bm{\Psi} + \lambda_t\idmatrix)^{-1} \nonumber\\ 
	& = (\bm{M}_1^{-1} + \bm{\Phi}_2\transpose\bm{\Phi}_2 )^{-1}(\bm{\Phi}_1\transpose\bm{Y}_1 + \bm{\Phi}_2\transpose\bm{Y}_2)\bm{\Psi}(\bm{\Psi}\transpose\bm{\Psi} + \lambda_t\idmatrix)^{-1}  \nonumber \\ 
 	& = [\bm{M}_1 - \bm{M}_1\bm{\Phi}_2\transpose(\idmatrix + \bm{\Phi}_2\bm{M}_1\bm{\Phi}_2\transpose )^{-1}\bm{\Phi}_2\bm{M}_1] \nonumber\\
	& \qquad(\bm{\Phi}_1\transpose\bm{Y}_1 + \bm{\Phi}_2\transpose\bm{Y}_2)\bm{\Psi}(\bm{\Psi}\transpose\bm{\Psi} + \lambda_t\idmatrix)^{-1}\label{woodburyIncremental} \,,
}
where $\bm{M}_1=(\bm{\Phi}_1\transpose\bm{\Phi}_1 + \lambda_d\idmatrix)^{-1}$. In Eq.~(\ref{woodburyIncremental}), we have made use of the Woodbury identity. Thus, in order to update the model parameters, the $l\times l$ matrix  $(\idmatrix + \bm{\Phi}_2\transpose\bm{M}_1\bm{\Phi}_2 )$ has to be inverted. A practical implementation for this scheme is given in Algorithm~\ref{twostepminibatchesinstances}. If $l>\odsize$, it is useful to make use of the matrix identity
\eq{
(\idmatrix + \bm{\Phi}_2\bm{M}_1\bm{\Phi}_2\transpose)^{-1}\bm{\Phi}_2\transpose \bm{M}_1= \bm{\Phi}_2\transpose(\bm{\Phi}_2\transpose\bm{\Phi}_2 +\bm{M}^{-1}_1)^{-1}\
}
in line \ref{dbigerl} so that only a $\odsize\times\odsize$ matrix has to be inverted. To update for new tasks, a very similar algorithm can be derived. In this case we assume that the data is divided as follows:
\eq{
\bm{\Psi} = 
\begin{pmatrix}
\bm{\Psi}_1\\
\bm{\Psi}_2
\end{pmatrix}  
\quad \mathrm{and}\quad
\bm{Y} = 
\begin{pmatrix}
\bm{Y}_1\ 
\bm{Y}_2
\end{pmatrix}\,,
}
with $\bm{\Phi}_2\in\mathbb{R}^{l\times\qdsize}$, $\bm{\Psi}_1\in\mathbb{R}^{(\qsize-l)\times\qdsize}$, $\bm{Y}_1\in\mathbb{R}^{\osize\times(\qsize-l)}$ and $\bm{Y}_2\in\mathbb{R}^{\osize\times l}$. Using this notation, Algorithm~\ref{twostepminibatchestasks} updates the parameters for a set of $l$ new tasks.

\begin{algorithm}[t]
  \begin{algorithmic}[1]
  \INPUT $\bm{\Phi}_\mathrm{new}, \bm{Y}_\mathrm{new}, \bm{W}_\mathrm{old}, \bm{M}_\mathrm{old}, \bm{B}^\gkernelf, \bm{\Phi}_\mathrm{old}\transpose\bm{Y}_\mathrm{old}$

\State $\bm{M}_\mathrm{new} = \bm{M}_\mathrm{old} - \bm{M}_\mathrm{old}\bm{\Phi}_\mathrm{new}\transpose(\bm{\Phi}_\mathrm{new}\bm{M}_\mathrm{old}\bm{\Phi}_\mathrm{new}\transpose + \idmatrix)^{-1}\bm{\Phi}_\mathrm{new}\bm{M}_\mathrm{old}$ \Comment{$\mathcal{O}(l^3 + d^2l)$} \label{dbigerl}
\State $\bm{W}_\mathrm{new}^\mathrm{TS} = \bm{M}_\mathrm{new} (\bm{\Phi}_\mathrm{old}\transpose\bm{Y}_\mathrm{old} + \bm{\Phi}_\mathrm{new}\transpose\bm{Y}_\mathrm{new})\bm{B}^\gkernelf$ \Comment{$\mathcal{O}(d^3 + q^3 )$} 
 \caption{Update primal parameters $\bm{W}_\mathrm{old}^\mathrm{TS}$ of two-step KRR using batches of $l$ new instances, assuming $l<d$. In addition to the old weight matrix, new instance features and labels, the algorithm requires precomputed matrices $\bm{M}_\mathrm{old}=(\bm{\Phi}_\mathrm{old}\transpose\bm{\Phi}_\mathrm{old} + \lambda_d\idmatrix)^{-1}$, $\bm{B}^\gkernelf=\bm{\Psi}(\bm{\Psi}\transpose\bm{\Psi} + \lambda_t\idmatrix)^{-1}$ and $\bm{\Phi}_\mathrm{old}\transpose\bm{Y}_\mathrm{old}$. The algorithm also updates these matrices using the new data.}
 \State \textbf{return} $\bm{W}_\mathrm{new}, \bm{M}_\mathrm{new}, (\bm{\Phi}_\mathrm{old}\transpose\bm{Y}_\mathrm{old} + \bm{\Phi}_\mathrm{new}\transpose\bm{Y}_\mathrm{new})$
  \label{twostepminibatchesinstances}
  \end{algorithmic}
\end{algorithm}

\begin{algorithm}[t]
  \begin{algorithmic}[1]
  \INPUT $\bm{\Psi}_\mathrm{new}, \bm{Y}_\mathrm{new}, \bm{W}_\mathrm{old}, \bm{N}_\mathrm{old}, \bm{B}^\kernelf, \bm{Y}_\mathrm{old}\bm{\Psi}_\mathrm{old}$

\State $\bm{N}_\mathrm{new} = \bm{N}_\mathrm{old} - \bm{N}_\mathrm{old}\bm{\Psi}_\mathrm{new}\transpose(\bm{\Psi}_\mathrm{new}\bm{N}_\mathrm{old}\bm{\Psi}_\mathrm{new}\transpose + \idmatrix)^{-1}\bm{\Psi}_\mathrm{new}\bm{N}_\mathrm{old}$ \Comment{$\mathcal{O}(l^3 + r^2l)$}
\State $\bm{W}_\mathrm{new}^\mathrm{TS} =\bm{B}^\kernelf (\bm{Y}_\mathrm{old}\bm{\Psi}_\mathrm{old}+\bm{Y}_\mathrm{new}\bm{\Psi}_\mathrm{new})\bm{N}_\mathrm{new}$ \Comment{$\mathcal{O}(d^3 + r^3 )$} 
 \State \textbf{return} $\bm{W}_\mathrm{new}, \bm{N}_\mathrm{new}, (\bm{Y}_\mathrm{old}\bm{\Psi}_\mathrm{old}+\bm{Y}_\mathrm{new}\bm{\Psi}_\mathrm{new})$

 \caption{Update primal parameters $\bm{W}_\mathrm{old}^\mathrm{TS}$ of two-step KRR using batches of $l$ new tasks, assuming $l<q$. In addition to the old weight matrix, new task features and labels, the algorithm requires precomputed matrices $\bm{N}_\mathrm{old}=(\bm{\Psi}_\mathrm{old}\transpose\bm{\Psi}_\mathrm{old} + \lambda_t\idmatrix)^{-1}$, $\bm{B}^\kernelf=(\bm{\Phi}\transpose\bm{\Phi} + \lambda_d\idmatrix)^{-1}\bm{\Phi}\transpose$ and $\bm{Y}_\mathrm{old}\bm{\Psi}_\mathrm{old}$. The algorithm also updates these matrices using the new data.}
  \label{twostepminibatchestasks}
  \end{algorithmic}
\end{algorithm}

\section{Experiments}\label{experimentalsection}

In the experiments we will demonstrate the learning properties of two-step KRR, compared to independent-task and Kronecker KRR. Furthermore, the experiments illustrate the efficient algorithms for training and evaluating the two-step KRR model, in contrast to the more limited toolkit for the Kronecker KRR model. To be more specific:
\begin{itemize}
\item In Section~\ref{DTreg} we study the performance of two-step and Kronecker KRR for the four different settings on four protein-ligands benchmarks. We illustrate the use of the analytical shortcuts for cross-validation.
\item In Section~\ref{comptransferlearning} we use a case study of protein-ligands interactions to study the differences between independent-task regression, multi-task learning and zero-shot learning. Some of the learning curves can only be made without resorting to slower gradient-based optimization for training the models with two-step KRR.
\item Finally, in Section~\ref{OLHTC} we demonstrate learning in mini-batches on a large-scale hierarchical text classification problem. 
\end{itemize}

We refer to~\cite{Romera-paredes2015} for some experimental results which show that two-step KRR is a competitive method compared to established zero-shot learning methods, including DAP~\citep{Lampert2014} and ZSRwUA~\citep{Jayaraman2014}, for some zero-shot learning benchmark datasets. We refer to that paper for a comparison with the state-of-the-art.

\subsection{Study of regularization for the different settings}\label{DTreg}

In this section we investigate the influence of the regularization parameters $\lambda$, $\lambda_d$ and $\lambda_t$, of two-step and Kronecker KRR for the different Settings A, B, C and D. We will also demonstrate the scalability of the different shortcuts for cross-validation, described in Section~\ref{computationalaspects}. To this end, we use four drug-target classification datasets collected by~\citet{Yamanishi2008}\footnote{Available at \url{http://web.kuicr.kyoto-u.ac.jp/supp/yoshi/drugtarget}}. Each dataset concerns a different class of protein targets: enzymes (e), G protein-coupled receptors (gpcr), ion channels (ic) and nuclear receptors (nc). The properties of these datasets are given in Table~\ref{YamanishiPLtable}. The interactions are given in the form of a binary adjacency matrix. Both the drugs and targets come along with a respective similarity matrix. For the drugs, common substructures are calculated using a graph alignment algorithm. The Jaccard similarity measure is used to obtain a drug similarity based on these substructures. The similarity matrix of the targets is a normalized version of the scores obtained by Smith-Waterman alignment. We rescored the labels, such that positive interactions have a value of $N/N^+$, while negative interactions have a value of $-N/N^-$, with $N$ the number of pairs and $N^+$ and $N^-$ the number of positive and negative interactions, respectively. By using this relabeling, minimizing the squared loss becomes equivalent with Fisher discriminant analysis~\citep{Bishop2006}, making our method more suitable for classification.

\begin{table}[t]
  \begin{center}
   \topcaption{Overview of the different drug-target datasets used in Section~\ref{DTreg}.} \label{YamanishiPLtable}
   \begin{tabular}{ l l l l l } 
\hline\hline
dataset & e & gpcr & ic & nr \\
\hline 
\# targets & 664 & 95 & 204 & 26 \\
\# drugs & 445 & 223 & 210 & 54 \\
fraction of interactions (\%) & 0.99 & 3.00 & 3.45 & 6.41 \\
median degree targets & 2 &3 & 5 & 3  \\
median degree drugs & 2 & 2 & 3 & 1\\
\hline
\hline
   \end{tabular}
   \end{center}
   
\end{table}

\begin{table}[t]
  \begin{center}
   \topcaption{Overview of the performance and running time using Kronecker and two-step KRR for the different protein-ligand datasets and cross-validation settings. One experiment could not be completed in less than three days of running time. See main text for details.} \label{TSKKCVcomparision}
\begin{tabular}{cc|cccc|rrrr}
\hline \hline
&&\multicolumn{4}{c|}{best performance (AUC)}&\multicolumn{4}{c}{running time}\\
\hline
\multicolumn{2}{c|}{data}&e&gpcr&ic&nr&e&gpcr&ic&nr\\
\hline
setting&method&&&&&&&&\\
\hline
A&KK&0.9640&0.9478&0.9723&0.8662&7.81s&0.22s&0.53s&0.01s\\
A&TS&0.9644&0.9420&0.9705&0.8857&66.42s&3.21s&5.6s&0.2s\\
B&KK&0.9196&0.8280&0.9495&0.7475&5671.1s&13.33s&91.97s&0.37s\\
B&TS&0.9256&0.8702&0.9507&0.7893&104.9s&10.28s&19.35s&2.09s\\
C&KK&0.8641&0.8742&0.8438&0.8250&2442.41s&82.67s&98.59s&1.28s\\
C&TS&0.8695&0.8772&0.8475&0.8515&94.39s&20.82s&22.36s&4.31s\\
D&KK&-&0.8228&0.7691&0.7107&$>$3d&2.14h&5.49h&26.09s\\
D&TS&0.8270&0.8319&0.7706&0.7275&66.6s&3.22s&5.77s&0.21s\\
\hline \hline
\end{tabular}
   \end{center}
   
\end{table}

For each of the models, we perform leave-one-out cross-validation for new pairs, new targets and both, as described in the introduction. The cross-validated predictions are obtained using the computation short-cuts described in Section~\ref{HOOshortcuts}. For the four different cross-validation settings we use AUC as a performance measure:

\begin{itemize}
\item For Settings A and D, i.e.\ imputation of pairs and new targets and ligands, we use the micro AUC. Here, the AUC is calculated over all the pairs and the model is evaluated for its ability to give a higher score to an arbitrary interacting dyad compared to an arbitrary non-interacting dyad.
\item For the setting where we leave out one target with all the drugs at a time (Setting~B), the AUC is calculated over all the drugs for each target and this then averaged over the drugs.
\item Likewise, for the setting with new drugs with all the ligands for validation (Setting~C), the AUC is calculated over all the ligands for each drug and then averaged over the targets.
\end{itemize}

\begin{figure}[t]
   \begin{center}
   \includegraphics[width=0.95\textwidth]{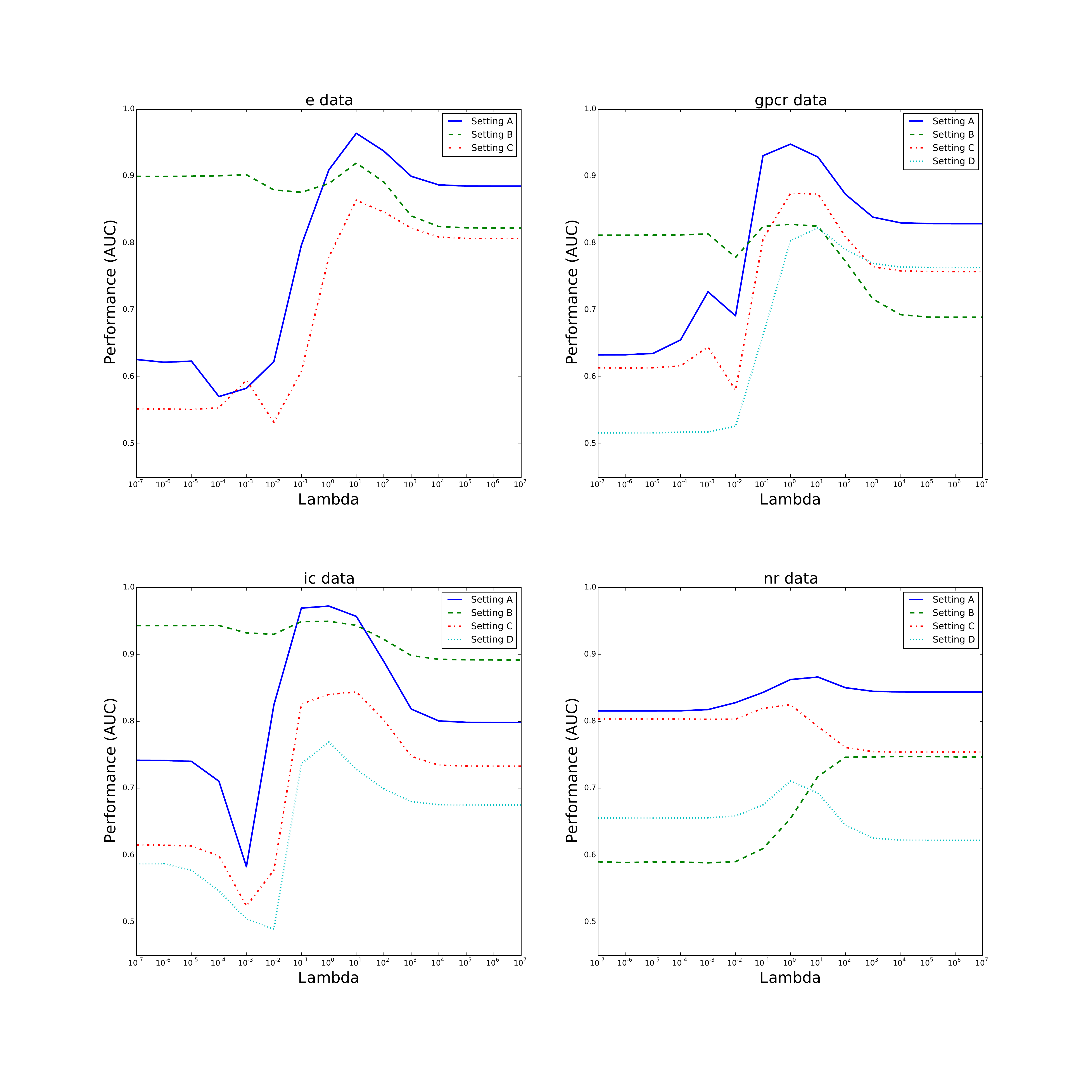} 
   \caption{Performance for different evaluation schemes for the four drug-target datasets for different values of the regularization parameter using Kronecker KRR. The optimal regularization is heavily dependent on the hyperparameter and validation setting. See online version for color.}
   \label{KKreg}
   \end{center}
\end{figure}

\begin{figure}[t]
   \begin{center}
   \includegraphics[width=0.95\textwidth]{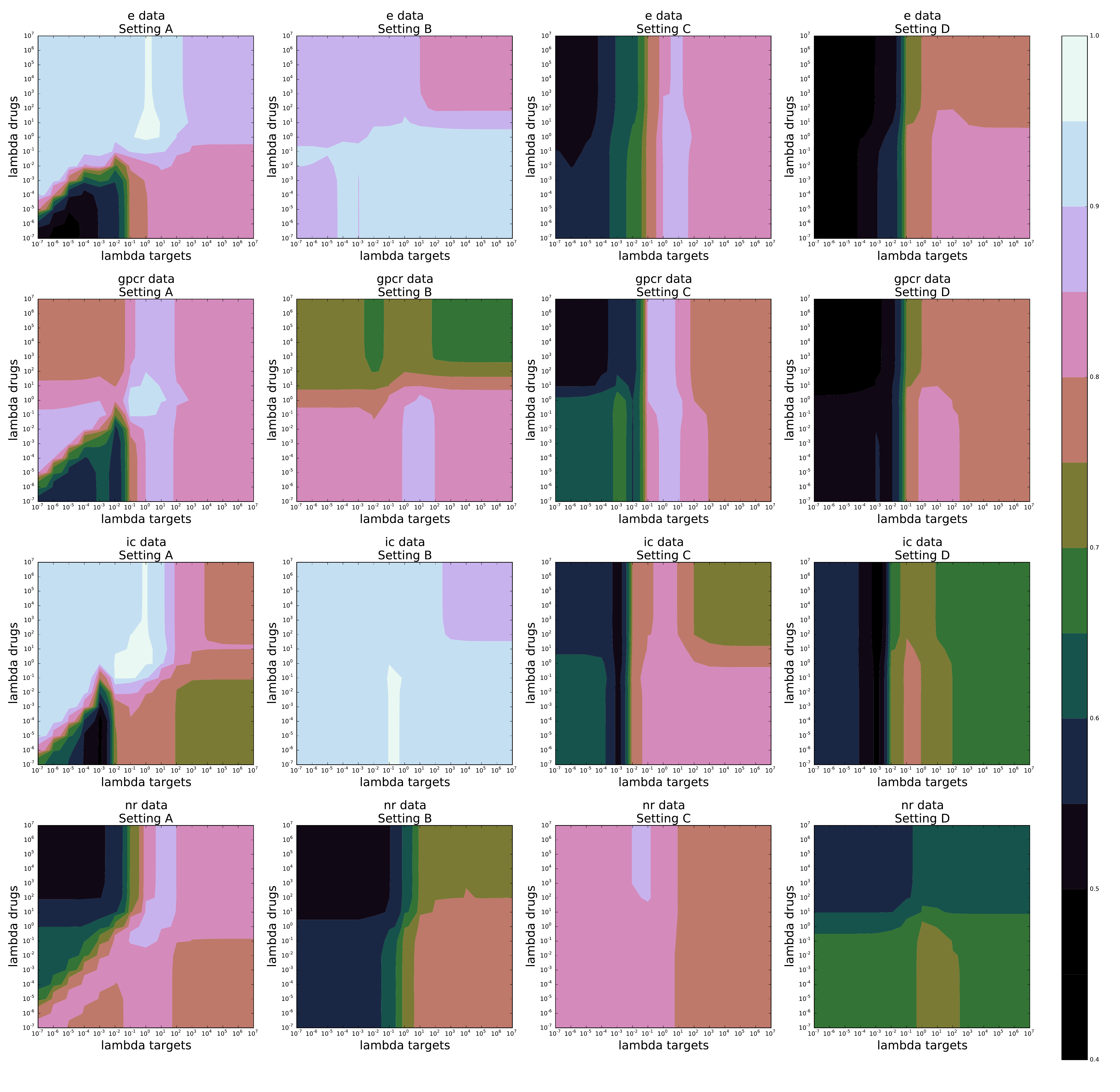} 
   \caption{Performance for different evaluation schemes for the four drug-target datasets for different values of the regularization parameters using two-step KRR. The optimal regularization is heavily dependent of the hyperparameters and validation setting. See online version for color.}
   \label{TSreg}
   \end{center}
\end{figure}

We trained Kronecker and two-step KRR with the regularization parameters $\lambda$, $\lambda_d$ and $\lambda_t$ each taken from the grid with values $10^{-7}, 10^{-6},\ldots, 10^5, 10^6$. For each (combination of) regularizer(s), the performance was calculated. For Setting A for both methods and the other settings for two-step KRR, we used the analytic shortcuts to calculate the holdout values. For the other settings using Kronecker KRR, we calculated the eigenvalue decomposition of a submatrix of the Gram matrix $\dkernelm$ ($\tkernelm$) for each row (column) for Setting B (Setting C). For Setting D, we computed the eigenvalue decomposition of a submatrix of $\dkernelm$ for each row of $\mathbf{Y}$, while calculating the eigenvalue decomposition of a submatrix of $\tkernelm$ for each element in this row. This corresponds to performing the smallest number of computations, while only holding one eigenvalue decomposition of each Gram matrix in memory.

Table~\ref{TSKKCVcomparision} shows the best performances for both methods, as well as the running times for performing the complete cross-validation and hyperparameter grid search. For the different settings and datasets, we observe that both methods have a similar performance, with two-step KRR often slightly outperforming Kronecker KRR. For both methods, Setting D is the hardest and Setting A the easiest, as expected. 

It is immediately clear that the optimal regularization depends both on the type of dataset, as well as on the type of predictions one is aiming at. For example, comparing the setting of new targets with new drugs for the datasets e, gpcr and ic, we see that for generalizing to new targets a low regularization for the drugs is needed, while the latter cases require the opposite. For these cases, it seems the models can predict better for new targets compared to drugs, indicating that the feature description for the drugs is more adequate. Hence, finding a suitable model for predicting for new targets in this case is harder, indicated by smaller regions of high performance. We see that for the nr data, the regions are quite irregular and dependent on the setting, likely because this is the smallest dataset. Depending on the dataset, Setting B or Setting C is the harder one. Likely this is determined by both the number of proteins and ligands in the training set, as well as the quality of the two kernel matrices. 

When comparing the running times for model selection, we observe the computational advantage of two-step KRR. For Setting A, both methods have a holdout shortcut, hence both are fast. Kronecker KRR has to iterate over a grid of 15 regularization values, while two-step KRR has to search a grid of 15 $\times$ 15 regularization parameters, making it slower. Both methods are very fast in practice though. For Settings B and C, there is only an efficient algorithm for two-step KRR. For datasets larger than dataset nr, two-step KRR is much faster than Kronecker KRR. Note that because for these settings we calculate the AUC for each row and column in each step, calculating the performance takes more computing time than calculating the leave-one-out values! If we would use the same performance metric for all settings (which would not make sense for the AUC), two-step KRR would have approximately the same running time for all settings. Finally, for Setting D two-step KRR is much faster compared to Kronecker KRR, where it was not even possible to do this cross-validation for the e dataset within three days.

In Figures~\ref{KKreg} and \ref{TSreg} it is visualized how the performance depends on the regularization for the different settings. The performance is quite sensitive to the value of the regularization parameter(s), a fact well known in machine learning. The optimal regularization is also strongly determined by the cross-validation setting, especially for the two-step method. The contour plots of Figure~\ref{TSreg} illustrate that for different settings a specific model has to be selected. Penalizing the instances and the tasks separately is natural for these types of learning problems, and the effect on the performance can be a valuable diagnostic tool to aid the model building. Two-step KRR allows for efficiently exploring this space for any setting, in contrast to Kronecker KRR.

\subsection{Comparison of different transfer learning settings}\label{comptransferlearning}

In this series of experiments, we compare different types of transfer learning settings in protein-ligand prediction. We simulate the zero-shot and almost zero-shot learning problem as follows. In each experiment, one ligand is considered to be the target task in question, where the task is to predict the interactions of proteins with respect to the target. Further, other tasks formed in the same way are provided as auxiliary information, leading to a zero-shot learning or almost zero-shot learning setting. The experiments are performed 100 times with different training/test set splits.

The experiments were performed using two different datasets. In the first experiment, the enzyme dataset of the previous section was reused. The goal is to learn to predict for the given ligand the binding with a set of proteins that were not encountered in the training phase. The performances are always computed over a test set of 164 protein targets for a given task, i.e.~we assess whether for a given target we can discriminate between proteins with more or less affinity for this ligand. The performance was measured by calculating the AUC for the test set of proteins for each ligand or task separately (i.e.~macro AUC). 

We also used a different drug-target interaction prediction dataset\footnote{\url{http://users.utu.fi/aatapa/data/DrugTarget}} \citep{davis2011comprehensive,Pahikkala2015} consisting of 68 drug compounds and 442 protein targets. In contrast to the earlier protein-ligand datasets, this is a regression task with real-valued labels. The kernel between the drugs is based on the 3D Tanimoto similarity coefficient, and the sequence similarity between the protein targets was computed using the normalized version of the Smith-Waterman score. Further, for each drug-protein pair we have a real-valued label, the negative logarithm of the kinase disassociation constant $K_d$, which characterizes the interaction affinity between the drug and target in question. In each experiment, the task of interest corresponds to one of the drugs in the data set. The goal is to learn to predict for the given drug the $K_d$ values for proteins unseen during the training phase. The performances are always computed over a test set of 192 protein targets for a given task, i.e.~we assess whether for a given target we can discriminate between proteins with more or less affinity for this drug. The performances are averages over all repetitions and over all target tasks, and are measured using the concordance index \citep{gonen2005concordance} (C-index), also known as the pairwise ranking accuracy:
\eq{\frac{1}{|\{(i,j)\mid y_i > y_j \}|}\sum_{y_i > y_j}H(\predfun_i - \predfun_j)\,,
}
where $y_i$ denotes the true and $\predfun_i$ the predicted label, and $H$ is the Heaviside step function. The C-index can be seen as a generalization of the area under the ROC curve (AUC). Regularization parameter selection is performed using leave-one-out cross-validation on the training data using the shortcuts of Section~\ref{HOOshortcuts}.
The algorithms used in the experiments are implemented in the RLScore open source machine learning library\footnote{Available at \url{https://github.com/aatapa/RLScore}}.

For each task, we vary the number of available training proteins. For the enzyme dataset the number of training proteins is increased from 25 to 500 in steps of 25 and for the drug-target affinity the number of proteins is varied from 10 to 250 in steps of 5. In addition, we have available the training data for all training proteins for the auxiliary tasks. As summarized in Figure~\ref{ref:relationMatrices}, we evaluate a number of different approaches:
\begin{figure}[t]
{\begin{center}
\includegraphics[width=0.99\linewidth]{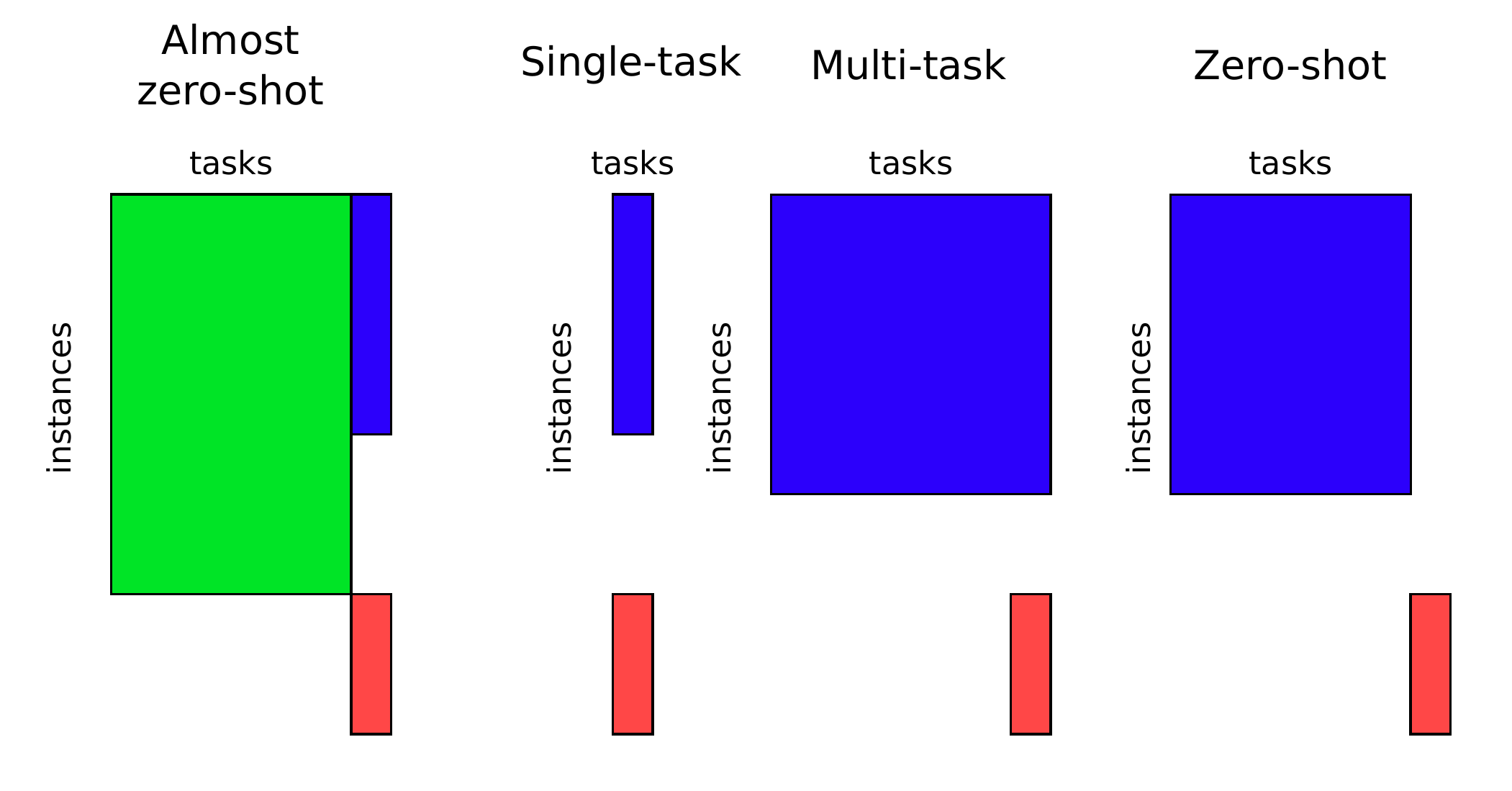}
\end{center}
}
\caption{Overview of the approaches investigated in Section~\ref{comptransferlearning}. Training data of which the size is constant in the experiments is colored in green. Training data of which the number of instances varies over different experiments is shown in blue and test data is indicated by red. See online version for color figure.}
\label{ref:relationMatrices}
\end{figure}
\begin{itemize}
\item Independent-task: KRR trained with data from the target task only. The regularization parameter is selected based on leave-one-out cross-validation for Setting A. 
\item Multi-task: both the target and auxiliary tasks have the same training data available (multi-output learning leveraging task correlations, tackled with Kronecker and two-step KRR). The regularization parameter is selected based on leave-one-out cross-validation for Setting A (Kronecker and two-step KRR) and Setting B (two-step KRR). 
\item Zero-shot learning: Kronecker KRR with no data for the target task using Kronecker and two-step KRR. The regularization parameter is selected based on leave-one-out cross-validation for Setting A (Kronecker and two-step KRR) and Setting D (two-step KRR).  
\item Almost zero-shot learning: using a varying amount of data from the target task, and all the available data from auxiliary tasks (tackled with two-step KRR). Here the missing labels for the target task are imputed in a first step and in a second step the `completed' data is used to predict for the new proteins. For both steps the regularization parameter is selected based on leave-one-out cross-validation for Setting~A.
\end{itemize}
We do not consider Kronecker KRR in the almost full cold start experiment due to computational considerations. Unlike for two-step KRR, no closed-form solution exists for the method in this setting, 
and the iterative conjugate gradient based method has rather poor scalability.
%
In Figure~\ref{fig:drugtarget}, we present the results for the experiments for the two datasets. The top two plots show the experiments where all the auxiliary tasks have the data for all training proteins, and the amount of data available for the target task is varied. It can be seen that learning is always possible even in the full zero-shot setting, where both two-step KRR and Kronecker KRR perform much better than random. For both datasets, the independent-task approach begins to outperform the full zero-shot setting after some point when one has access to enough training proteins. Combining these two sources of information leads to the best performance for the enzyme dataset and for the drug-target affinity dataset up until 150 training proteins. In both cases, using auxiliary tasks can greatly improve performance when there are only a few labeled instances of the target task. However, once there is enough data available for the target task, there is no longer any positive transfer from the auxiliary tasks, though there also seems to be no harm from negative transfer.

In the second row of Figure~\ref{fig:drugtarget} we consider a setting in which there is the same amount of data available for both the auxiliary tasks and the target tasks. This setting corresponds closely to the traditional multi-output regression problem, the exception being that only the label for the target task is of interest during testing. For the enzyme dataset, Kronecker KRR slightly outperforms two-step KKR and both perform better than independent-task KRR. For the two-step method, it does not seem to matter whether the hyperparameters were selected for Setting A or Setting B. For the second dataset, we can see that the multi-task method that uses the task correlation information fails to outperform the simple independent-task method, suggesting that on this type of data one requires significantly more data in the auxiliary tasks compared to the target tasks in order for it to be helpful for learning.

In the bottom row of Figure~\ref{fig:drugtarget} we consider the full zero-shot learning setting, while increasing the amount of data available for the auxiliary tasks. For the first dataset we observe that the two-step method slightly outperforms Kronecker KRR when the hyperparameters are optimized for Setting D. For the second dataset, there is no noteworthy difference between both methods. Both approaches generalize to the unknown target task, though the results are still much worse than when having a significant amount of data for the target task.

\begin{figure}[t]
  \centering
  \subfloat[Enzyme dataset.]{\includegraphics[width=0.5\textwidth]{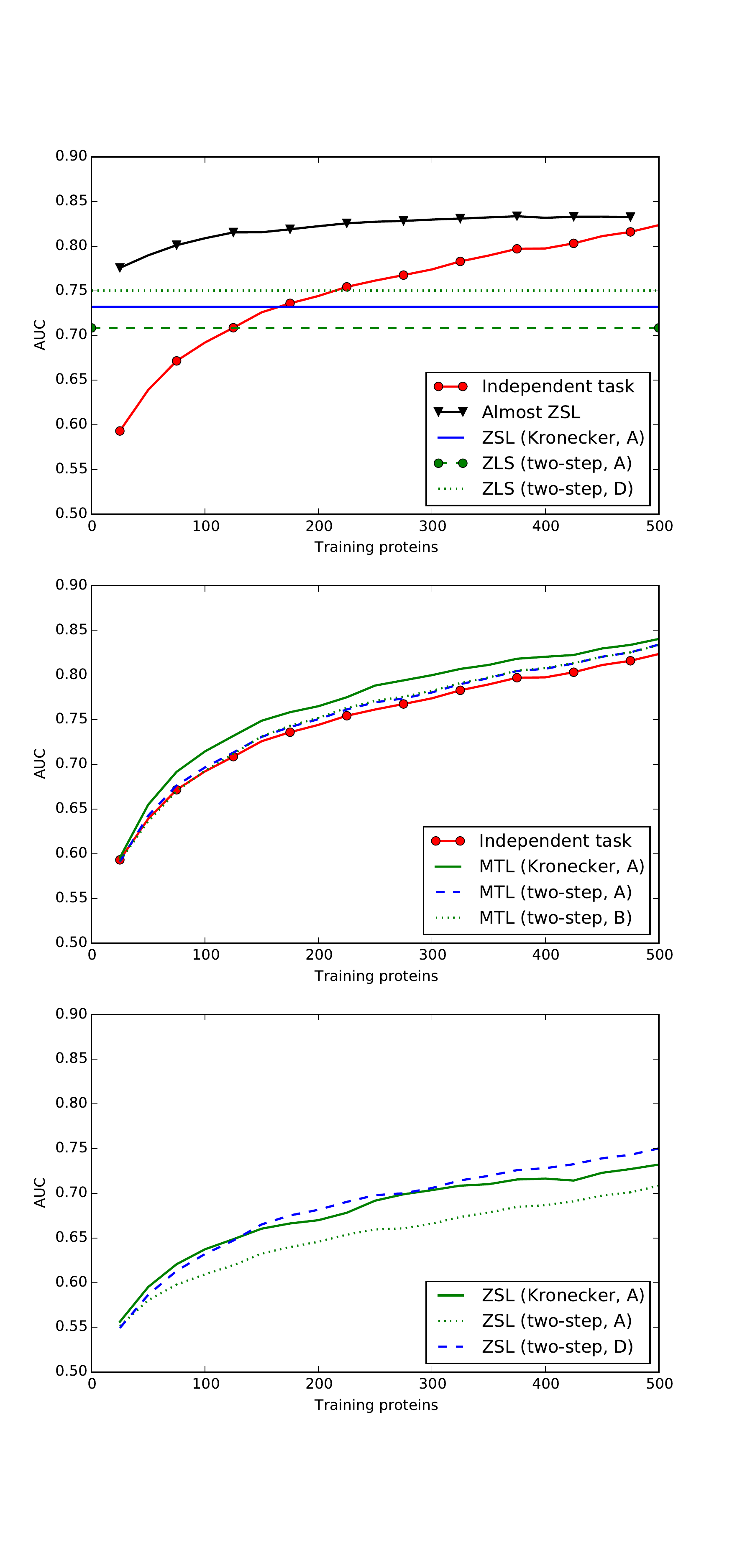}\label{fig:f1}}
  \hfill
  \subfloat[Drug-target affinity dataset.]{\includegraphics[width=0.5\textwidth]{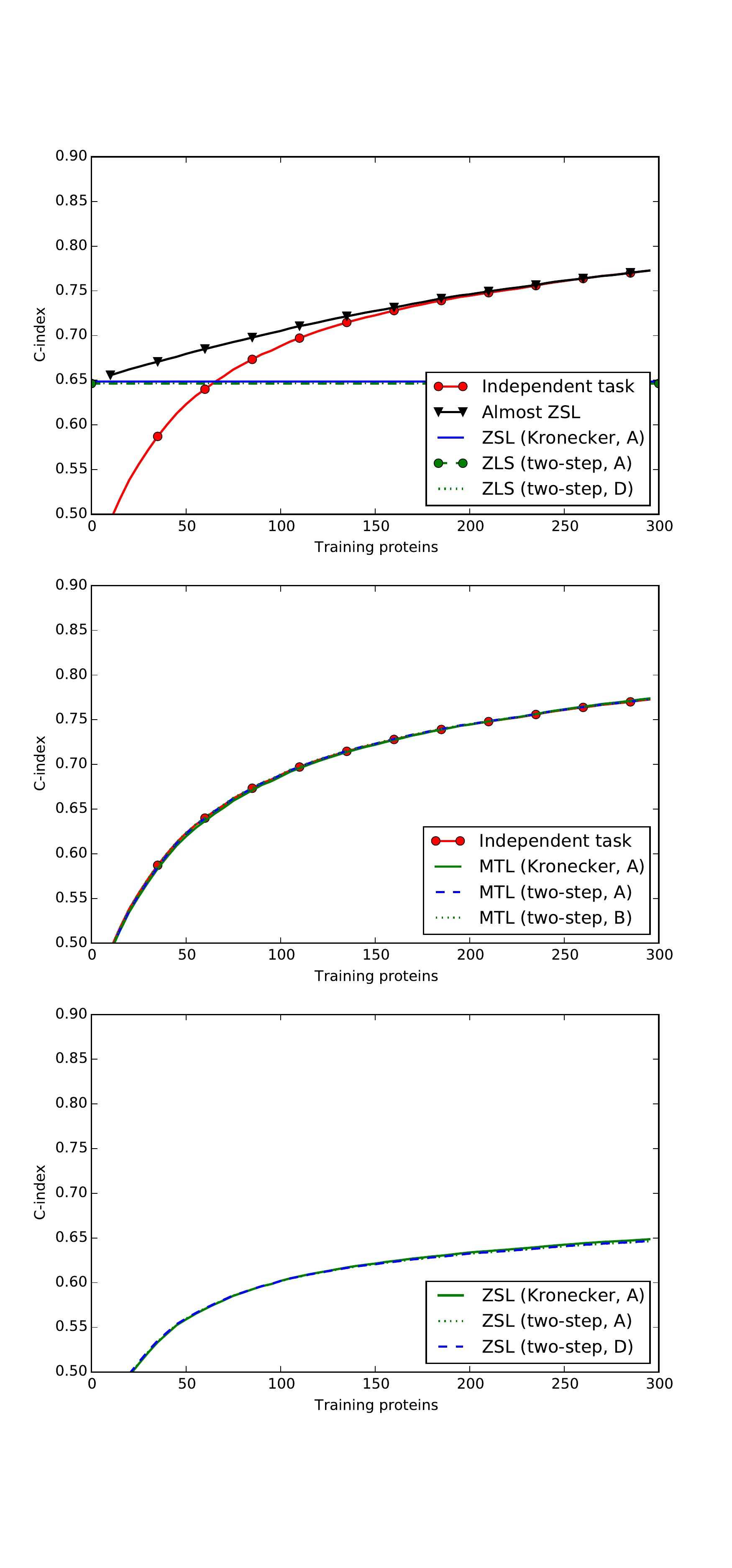}\label{fig:f2}}
  \caption{Learning curves for the (a) the enzyme and (b) the drug-target affinity data. Top: target data increased, Middle: target and auxiliary data increased, Bottom: auxiliary data increased.}
\label{fig:drugtarget}
\end{figure}

Here, two-step KRR shows itself to be competitive compared to Kronecker KRR. Previously, \citet{Schrynemackers2013} have, in their overview article on dyadic prediction in the biological domain, made the observation that in terms of predictive accuracy there does not seem to be a clear winner between the independent-task and multi-task type of learning approaches. Based on these experiments, a deciding factor on whether one may expect positive transfer from related tasks seems to be the amount of data available for the target task. The two-step method performs well in the almost full zero-shot settings with availability of a significant amount of auxiliary data and only very little data for the target task. But when there is enough data available for the target task, auxiliary data is no longer helpful.

\subsection{Online-learning of hierarchical text classification}\label{OLHTC}

In a final experiment we study the online learning of a large-scale hierarchical text classification problem. We will demonstrate learning with mini-batches, showing that both independent-task and two-step KRR improve when iteratively adding more training data. We used the Wikipedia benchmark dataset~\citep{Partalas2015} of the Large Scale Hierarchical Text Classification Challenge\footnote{http://lshtc.iit.demokritos.gr/}. We used the dataset provided for the third track, which is a subset of a larger, better curated set of another challenge. This dataset contains over 380,000 Wikipedia articles, for which the goal is to assign one of 12,633 labels, denoting the category of the article. The articles are described by a sparse bag-of-words vector with a length of 833,482. Each article can belong to only one class and the classes are part of a directed acyclic graph, representing the hierarchy of the categories (e.g.~`restricted Boltzmann machine' is a subcategory of `artificial neural networks', which is a subcategory of `machine learning'). For this subchallenge, the label space has been extended with 3000 novel labels. There are 6000 new articles that are labeled according to this new scheme. Here we will study the performance for both a test set of 10,000 articles with tasks seen in the training phase (Setting B or multi-task learning) and the performance on this dataset with novel labels (Setting D or zero-shot learning) as a function of the number of training articles. 

The bag-of-words representation was compressed by using canonical correspondence analysis to obtain 1000 orthogonal components that are maximally correlated with the training labels. We also added a dummy feature so that the model could fit an intercept. For describing the tasks, we considered two approaches:
\begin{itemize}
\item Considering all the tasks independently, i.e.~using independent-task KRR or, equivalently, $\tkernelm=\idmatrix$ for two-step KRR.
\item Using features describing the hierarchy between the classes. We used the Dijkstra algorithm to generate all pairwise distances $d_{ij}$ between the nodes of the label graph. A kernel was constructed by using a radial basis kernel with the distance in the exponent: $k_{ij} = e^{-\frac{d_{ij}}{10}}$.
\end{itemize} 
The model is initially trained using 5000 instances. Subsequently, the model is iteratively updated with batches of 1000 instances using Algorithm~\ref{twostepminibatches}. The regularization parameters $\lambda_d$ and $\lambda_d$ are chosen by minimizing the mean squared error\footnote{We did model selection on mean squared error rather than accuracy under the curve to speed up this procedure.} during leave-one-out cross-validation for Setting B (see Section~\ref{HOOshortcuts}).

\begin{figure}[t]
   \centering
   \includegraphics[width=\textwidth]{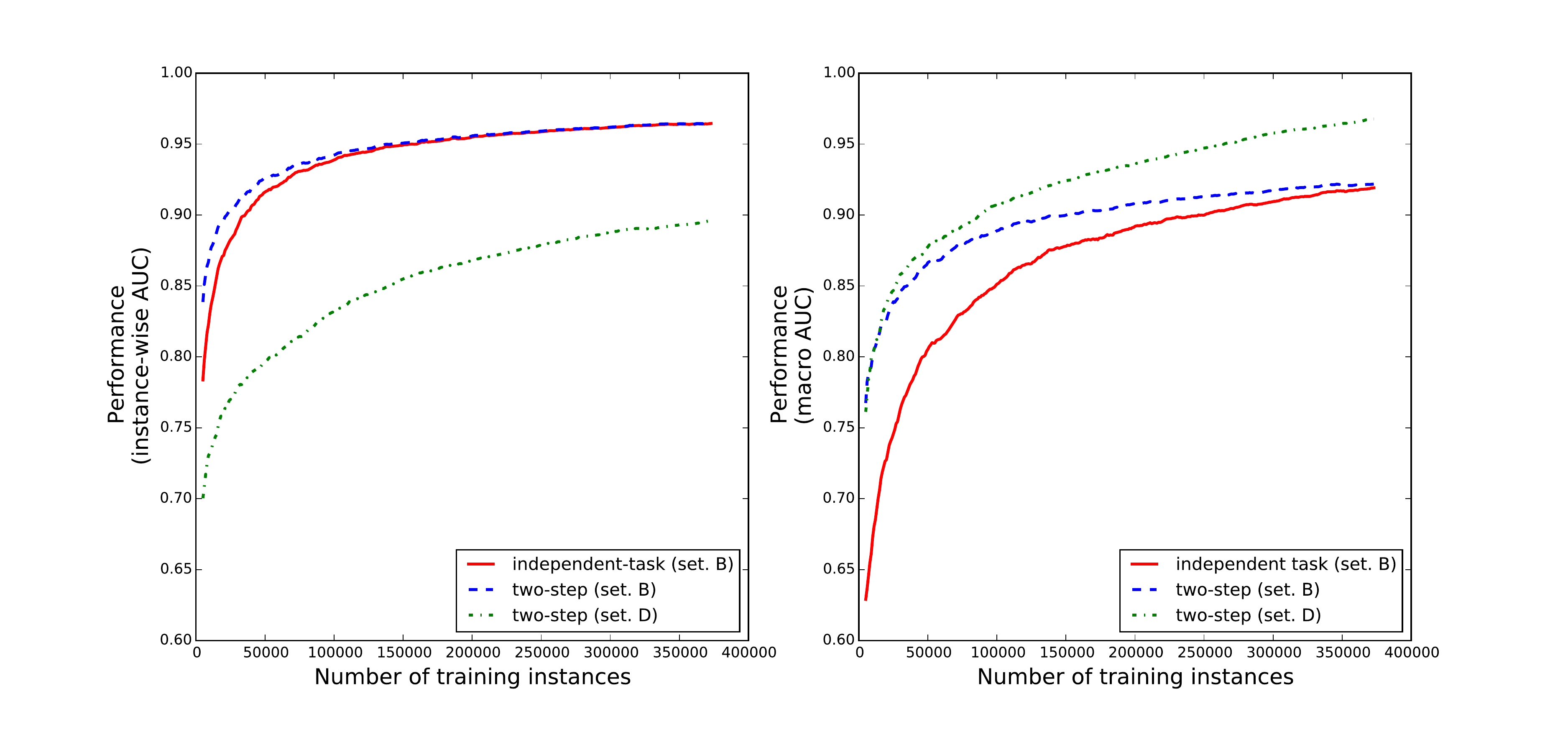} 
   \caption{Learning curves for the different models for the Wikipedia dataset.}
   \label{wikipediaLearningCurve}
\end{figure}

The accuracy of the predictions was measured either as AUC over the instances (i.e.~the capacity of the model to discriminate between a relevant and irrelevant label for a given article) and the macro AUC (i.e.~the capacity of the model to give a higher score to articles of a class compared to articles that are not of that class). These learning curves are presented in Figure~\ref{wikipediaLearningCurve}. For Setting B, for both evaluation schemes, using task features greatly improves the performance when the number of training instances is low. With more training data, both methods converge to a similar performance. Note that for the macro AUC evaluation, the learning rate is much slower compared to instance-wise AUC, implying that the former task is harder. Note that using two-step kernel ridge regression, the performance is never worse than for ridge regression. It is thus advisable to start with this method when few data points are available and update using Algorithm~\ref{twostepminibatchesinstances} when more labeled instances become accessible.

For Setting D (Figure~\ref{wikipediaLearningCurve}), we notice a different pattern. First note that the AUC values cannot be directly compared with those from Setting B, since the test set contains less and different labels. Here, for both evaluations, the test performance has not converged yet, implying that more training data would be beneficial for the model. This makes sense, as assigning novel labels is much harder than known labels. Nevertheless, using two-step kernel ridge regression one can both discriminate between relevant and irrelevant labels for a given article and between relevant and irrelevant articles for a given new label. For the macro AUC evaluation, the slope of the learning curve is again quite steep, indicating that for this problem more training data is required.

\section{Conclusions and perspectives}

In this work we have studied a new two-step kernel ridge regression method, in comparison to independent-task kernel ridge regression and Kronecker kernel ridge regression. We have shown that these methods are very related: Kronecker KRR is a special case of ordinary kernel ridge regression, two-step KRR is a special case of (Kronecker) KRR, while independent-task KRR is again a special case of two-step KRR. This unifying framework has allowed us to study the spectral interpretation for all these methods. Two-step KRR, which was found both theoretically and experimentally to work very well, has additional computational advantages. Because the model building can conceptually be decomposed into two independent steps, efficient novel hold-out tricks and algorithms for online learning can be obtained. Using the shortcuts we were experimentally able to study the learning rate of our methods. All experiments illustrate that the use of task features can significantly improve performance, but careful tuning is required.  An intriguing question for further research is whether it would be useful to combine other models than ridge regression or to even mix different types of spectral filters for the two steps.

Given the recent surge into fields such as zero-shot learning and extreme classification, two-step KRR has potential to become a standard tool for many problems. We believe that two-step KRR, as a special case of Kronecker KRR, is particularly useful in two specific situations. Firstly, when dealing with rather small-scale interaction datasets (hundreds or thousands rows and columns) with a lot of domain knowledge. Such situations are often encountered in biological applications, e.g.\ molecular interaction prediction or the study of species-species interactions. In these domains, kernel-based methods are already well established to encode prior knowledge. For such datasets, two-step KRR allows for fast and flexible model selection and validation, so that the researcher can easily explore what is possible with the data at hand. A second application would be in large-scale data applications. When dealing with huge output spaces, two-step KRR is a simple method to enforce prior knowledge on the outputs, while the suggested learning in mini-batches is an attractive alternative for gradient-based optimization.

\section*{Acknowledgements}
Part of this work was carried out using the Stevin Supercomputer Infrastructure at Ghent University, funded by Ghent University, the Hercules Foundation and the Flemish Government - department EWI.
\vskip 0.2in
\bibliographystyle{abbrv}
\bibliography{library,myBibliography}

\end{document}